\documentclass{article}
\usepackage[accepted]{icml2020}

\usepackage{ulem}
\renewcommand{\emph}[1]{\textit{#1}}

\usepackage{algorithm}
\usepackage{amsmath}
\usepackage{amsfonts}
\usepackage{amssymb}
\usepackage{bm}
\usepackage{wrapfig}
\usepackage{graphicx}
\usepackage[english]{babel}
\usepackage{mathdef}
\usepackage{multirow}

\usepackage[dvipsnames]{xcolor}
\definecolor{darkblue}{rgb}{0, 0, 0.5}

\usepackage[colorlinks=true,linkcolor=darkblue,citecolor=darkblue,urlcolor=darkblue]{hyperref}

\usepackage{amsthm}

\newtheorem{theorem}{Theorem}
\newtheorem{assumption}{Assumption}
\newtheorem{lemma}[theorem]{Lemma} 
\newtheorem{proposition}[theorem]{Proposition} 
\newtheorem{corollary}[theorem]{Corollary}

\newcommand{\coloneqq}{:=}
\global\long\def\th{\textbf{\ensuremath{\boldsymbol{\theta}}}}%
\global\long\def\x{\boldsymbol{x}}
\global\long\def\z{\boldsymbol{z}}
\global\long\def\E{\mathbb{E}}%
\global\long\def\R{\mathbb{R}}%
\global\long\def\lip{\text{Lip}}%
\global\long\def\supp{\text{supp}}%
\global\long\def\bp{\boldsymbol{\phi}}%
\global\long\def\M{\mathcal{M}}%
\global\long\def\w{\boldsymbol{w}}%
\global\long\def\y{\textbf{y}}%
\global\long\def\D{\mathcal{D}}%
\global\long\def\grad{\boldsymbol{g}}%
\global\long\def\DD{\mathcal{L}}%
\global\long\def\u{\boldsymbol{u}}%

\icmltitlerunning{Greedy Subnetwork Selection}

\begin{document}

\twocolumn[

\icmltitle{
Good Subnetworks Provably Exist: Pruning via Greedy Forward Selection
}
\icmlsetsymbol{equal}{*}
\begin{icmlauthorlist}
\icmlauthor{Mao Ye}{ut}
\icmlauthor{Chengyue Gong}{equal,ut}
\icmlauthor{Lizhen Nie}{equal,uchicago}
\icmlauthor{Denny Zhou}{go}
\icmlauthor{Adam Klivans}{ut}
\icmlauthor{Qiang Liu}{ut}
\end{icmlauthorlist}
\icmlcorrespondingauthor{Mao Ye}{my21@cs.utexas.edu}
\vskip 0.3in

\icmlaffiliation{ut}{Department of Computer Science, the University of Texas, Austin}
\icmlaffiliation{uchicago}{Department of Statistics, the University of Chicago}
\icmlaffiliation{go}{Google Research}
]
\printAffiliationsAndNotice{\icmlEqualContribution}
\date{\today}

\begin{abstract}
Recent empirical works show that large deep neural networks are often highly redundant and one can find much smaller subnetworks 
without a significant drop of accuracy. 
However, most existing 
methods of network pruning are empirical and heuristic, 
leaving it open whether good subnetworks provably exist, how to find them efficiently, and if network pruning can be provably better than direct training using gradient descent. 
We answer these problems positively by proposing a simple greedy  selection approach for finding 
good subnetworks, which starts from an empty network and greedily adds important neurons from the large network. 
This differs from
the existing 
methods based on 
backward elimination, which remove redundant neurons from the large network. 
Theoretically, 
applying the greedy selection strategy  
on sufficiently large {pre-trained} networks
guarantees to find small subnetworks with lower loss than networks 
directly trained with gradient descent. 
Our results also apply to pruning randomly weighted networks.
Practically, we improve prior arts of network pruning on learning compact neural architectures on ImageNet, including ResNet, MobilenetV2/V3, and ProxylessNet. 
Our theory and empirical results on MobileNet suggest  that we should fine-tune the pruned subnetworks to leverage the information from the large model, instead of re-training from new random initialization as suggested in \citet{liu2018rethinking}. 
\end{abstract}

\section{Introduction}
The last few years have witnessed 
the remarkable success of 
large-scale deep neural networks (DNNs) 
in achieving human-level accuracy on complex cognitive tasks, including image classification~\citep[e.g.,][]{he2016resnet}, speech recognition~\citep[e.g.,][]{amodei2016speech} and machine translation~\citep[e.g.,][]{wu2016google}. 
However,  modern large-scale DNNs tend to 
suffer from slow inference speed 
and high energy cost, which form critical bottlenecks on edge devices such as mobile phones and Internet of Things (IoT)~\citep{cai2018proxylessnas}.
It is of increasing importance to obtain 
DNNs with small sizes and low energy costs. 

Network pruning has been shown to be a successful approach for learning small and energy-efficient  neural networks~\citep[e.g.,][]{han2015compress}. 
These methods start with a pre-trained large neural network  
and  remove the redundant units (neurons or filters/channels) to obtain a much smaller subnetwork without significant drop of accuracy. 
See e.g., \citet[][]{zhuang2018dpl, luo2017thinet, liu2017slim, liu2018rethinking, he2019filter, he2018amc} for examples of recent works. 

However, despite the recent empirical successes,
thorough 
theoretical understandings on why and how  network pruning works are still largely missing. 
Our work is motivated by the following basic questions:  

\noindent\textbf{The Subnetwork Problems:}
\emph{
Given a pre-trained large (over-parameterized) neural network, 
does there exist a small subnetwork inside the large network that performs almost as well as the large network? 
How to find such a good subnetwork computationally efficiently?
Does the small network pruned from the large network provably outperform the networks of same size but directly trained with gradient descent starting from scratch?
}

\begin{figure*} \label{fig:vs}
    \centering
    \begin{tabular}{cc}
        \includegraphics[width=1.6\columnwidth]{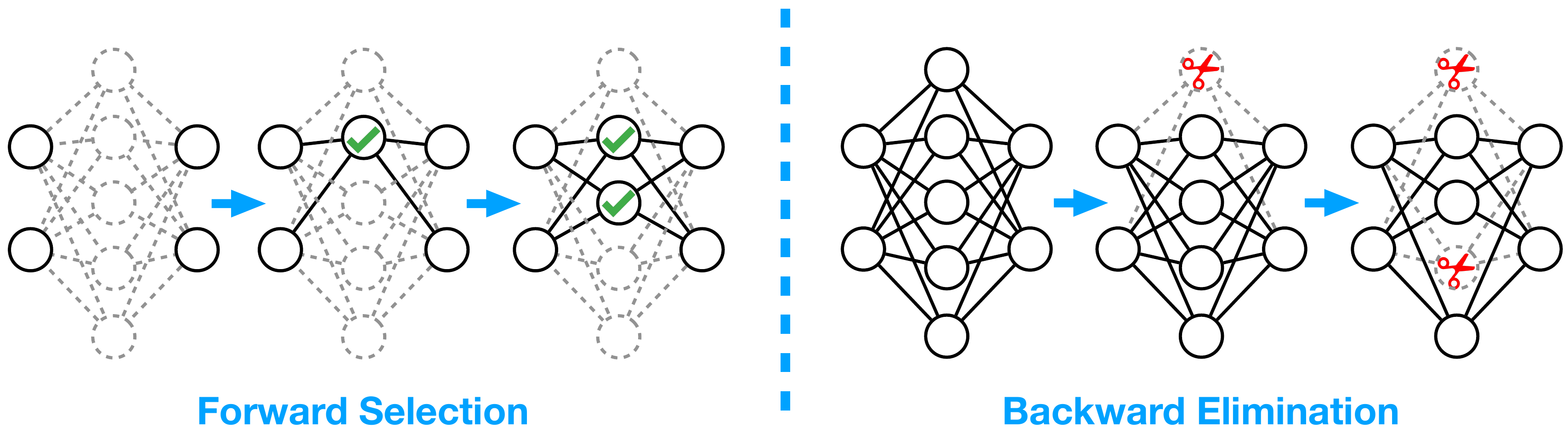} \\
        \end{tabular}
    \caption{Left: Our method constructs good  subnetworks by greedily adding the best neurons starting from an empty network. 
    Right: Many existing methods of network pruning works by gradually removing the redundant neurons starting from the original large network.  
    }
    \label{fig:my_label}
\vspace{-15pt}
\end{figure*}

We 
approach this problem by considering a simple greedy selection strategy, 
which starts from an empty network and 
constructs a good subnetwork by sequentially adding neurons from the pre-trained large network to yield the largest immediate decrease of the loss (see 
Figure~\ref{fig:vs}(left)). 
This simple algorithm provides both strong  theoretical guarantees and state-of-the-art empirical results, as summarized below.  

\paragraph{Greedy Pruning Learns Good Subnetworks}  
For two-layer neural networks, 
our analysis shows that our method  
yields a network of size  $n$
with a loss of $\mathcal O(1/n) + \mathcal L_N^*$, 
where $\mathcal L^*_N$ is the optimal loss 
we can achieve with all the neurons 
in the pre-trained large network of size $N$.    
Further, 
if the pre-trained large network is sufficiently over-parametrized, 
we achieve a much smaller  loss of  $\mathcal O(1/n^2)$. 
Additionally, the $\mathcal O(1/n^2)$  rate holds even when the weights of the large network are drawn i.i.d. from a proper distribution.

In comparison, standard training of networks of size $n$ by gradient descent yields 
a loss of $\mathcal O(1/n + \varepsilon)$ 
following the mean field analysis of \citet{song2018mean, mei2019meandimfree}, 
where $\varepsilon$ is usually a small term 
involving the loss of training infinitely wide networks; 
see Section~\ref{subsec:faster_mf} for more details. 

Therefore, our fast $\mathcal O(1/n^2)$ rate suggests that pruning from over-parameterized models 
guarantees to find  
more accurate small networks than direct training 
using gradient descent, 
providing a theoretical justification of the widely used network pruning paradigm.   

\paragraph{Selection vs. Elimination} 
Many of the existing  
methods of network pruning  
are based on \emph{backward elimination} of the redundant neurons starting from the full large network following certain criterion  \citep[e.g.,][]{luo2017thinet,liu2017slim}.
In contrast, our method is based on \emph{forward selection}, 
progressively growing the small network by adding the neurons; see Figure \ref{fig:vs} for an illustration. 
Our empirical results show that, 
our forward selection achieves better accuracy on pruning DNNs
under fixed FLOPs constraints,   
 e.g., ResNet \citep{he2016resnet}, MobileNetV2 \citep{sandler2018mobilenetv2}, 
ProxylessNet \citep{cai2018proxylessnas} and MobileNetV3 \citep{howard2019searching} on ImageNet.
In particular, our method outperforms all prior arts on pruning MobileNetV2 on ImageNet, achieving the best top1 accuracy under any FLOPs constraint.

Additionally, we draw thorough comparison between the forward selection strategy with the backward elimination in Appendix~\ref{apx: gbe}, and demonstrate the advantages of forward selection from both theoretical and empirical perspectives.

\paragraph{Rethinking the Value of Network Pruning} 
Both our theoretical and empirical discoveries 
highlight the benefits of using  
\emph{large, over-parameterized} 
models to learn small models that \emph{inherit the weights} of the large network. This implies that in practice, we should \emph{finetune} the pruned network to leverage the valuable information of both the structures and parameters in the large pre-trained model.

However, 
these observations are different from 
the recent findings of \citet{liu2018rethinking},
whose empirical results suggest that training a large, over-parameterized network is often not necessary 
for obtaining an efficient small network   
and finetuning the pruned subnetwork is no   better than retraining it starting from a new random initialization.

{\color{black}We think the apparent inconsistency happens because, different from our method, the pruning algorithms tested in \citet{liu2018rethinking} are not able to make the pruned network efficiently use the information in the weight of the original network. 
 To confirm our findings, 
we perform tests on compact networks on mobile settings such as 
MobileNetV2 \citep{sandler2018mobilenetv2} and MobileNetV3 \citep{howard2019searching}, 
and find that 
finetuning a pruned MobileNetV2/MobileNetV3 gives much better performance than 
re-training it from a new random initialization, 
which violates  
the conclusion of  \citet{liu2018rethinking}.    
Besides, 
we observe that  
increasing the size of  
pre-trained large models yields better  pruned subnetwork as predicted by our theory.   
See Section~\ref{sec:rethinking} and \ref{sec:large} for a thorough discussion. }

\paragraph{Notation}
We use notation $[N]:=\{1,\ldots, N\}$ for the set of the first $N$ positive integers. 
All the vector norms $\norm{\cdot}$ are assumed to be $\ell_2$ norm. 
$\left\Vert \cdot \right\Vert _{\lip}$ and $\left\Vert \cdot \right\Vert _{\infty}$ denote Lipschitz and $\ell_\infty$ norm for functions. 
We denote $\supp(\rho)$ as the support of distribution $\rho$.

\section{Problem and Method} \label{sec:problemset}

We focus on 
two-layer networks for analysis. 
Assume we are 
given a pre-trained large neural network consisting of $N$ neurons,
\begin{align*} 
f_{[N]}(\x) = \sum_{i=1}^N \sigma(\x; \th_i)/N,  
\end{align*}
 where 
 $\sigma(\x; \th_i)$ denotes the $i$-th neuron with parameter $\th_i \in\RR^d$ and input $\x$. 
 In this work, we consider 
 $$\sigma(\x; \th_i)=b_i\sigma_{+}(\boldsymbol{a}_i^\top\x),$$ 
 where $\th_i=[\boldsymbol{a}_i, b_i]$
 and  $\sigma_{+}(\cdot)$ is an activation function such as Tanh and ReLU. 
 But our algorithm works for general forms of  $\sigma(\x; \th_i)$. 
 Given an observed dataset $\dataset := (\x^{(i)}, y^{(i)})_{i=1}^m$ with $m$ data points, 
we consider the following regression loss of network $f$: 
 \begin{eqnarray*} \label{equ:mse}
 \DD[f] &=& \E_{(\x,y)\sim \dataset}[ \left( 
 f(\x) - y\right)^2]/2.
 \end{eqnarray*}
 
 We are interested in finding a subset $S$ of $n$ neurons ($n<N$) from the large network, which minimizes the loss of the subnetwork $f_S(\x) = \sum_{i\in S} \sigma(\x; \th_i)/|S|$, i.e., 
 \begin{align}\label{equ:minF}
 \min_{S \subseteq [N]}\DD[f_S]~~~~s.t.~~~~|S|\leq n.
 \end{align}
 {\color{black}Here we allow the set $S$ to  contain repeated elements.} This is a challenging combinatorial optimization problem. We propose a greedy forward selection strategy, which starts from an empty network 
 and 
 gradually adds the neuron that yields the best immediate decrease on loss. 
 Specifically, starting from $S_0 = \emptyset$, we sequentially add neurons via 
 \begin{equation} \label{equ:step}
 S_{n+1} \gets S_n \cup i_n^*
 ~~~~\text{where}~~~
 i_n^* = \argmin_{i\in [N]} \DD[f_{S_n \cup i}]. 
 \end{equation}
Notice that the constructed subnetwork inherits the weights of the large network and in practice we may further finetune the subnetwork with training data. More details of the practical algorithm and its extension to deep neural networks are in Section \ref{sec:algo}. 

\section{Theoretical Analysis} \label{sec:theory}
The simple greedy procedure 
yields strong theoretical guarantees, 
which, as a byproduct, 
also implies the existence of small and accurate subnetworks. 
Our results are two fold: 

\emph{i) Under mild conditions, 
the selected subnetwork of size $n$  achieves $\DD[f_{S_n}] =  \mathcal O(1/n) + \mathcal L_N^*$, 
where $\mathcal L_N^*$  is the best possible loss achievable by convex combinations of \emph{all} the $N$ neurons in $f_{[N]}$.
}

\emph{ii) We achieve a faster rate of $\mathcal L[f_{S_n}] = \mathcal O(1/n^2)$ if the large network $f_{[N]}$ is sufficiently over-parameterized and can overfit the training data subject to small perturbation 
(see Assumption~\ref{asm:inter_N}). 
}

In comparison, the mean field analysis of \citet{song2018mean, mei2019meandimfree} shows that:

\emph{iii) Training a network of size $n$ using (continuous time) gradient descent starting from random initialization gives an $\mathcal O(1/n+\varepsilon)$ loss, where $\varepsilon$ 
is a (typically small) term involving the loss of  infinitely wide networks trained with gradient dynamics.
See \citet{song2018mean, mei2019meandimfree} for details.  
}


Our fast $\mathcal O(1/n^2)$ rate 
shows that subnetwork selection from \emph{large, over-parameterized} models 
yields provably better results than 
training 
{small networks} of the same size 
 starting from scratch using gradient descent. 
This provides the first theoretical justification of the empirical successes of the popular network pruning paradigm. 

We now introduce the theory in depth. 
We start with the general  $\mathcal O(1/n)$ rate 
in Section \ref{subsec:theory_insuff}, 
and then establish and discuss the faster $\mathcal O(1/n^2)$ rate in Section \ref{subsec:faster_N} and \ref{subsec:faster_mf}.

\subsection{General Convergence Rate} \label{subsec:theory_insuff}
Let $\DD_N^*$ be the minimal loss achieved by the best convex combination of all the $N$ neurons in $f_{[N]}$, that is, 
\begin{eqnarray} \label{equ:LN*}
\DD_{N}^{*}=\underset{\boldsymbol{\alpha}=[\alpha_{1},...,\alpha_{N}]}{\min}\left\{ \DD[f_{\boldsymbol{\alpha}}]:\alpha_{i}\ge0,\sum_{i=1}^{N}\alpha_{i}=1\right\},
\end{eqnarray}
where $f_{\boldsymbol{\alpha}}=\sum_{i=1}^{N}\alpha_{i}\sigma(\th_{i},\x)$. It is obvious that $\DD_N^*\le\DD[f_{[N]}]$.
We can establish 
the general $\mathcal O(1/n)$ rate with the following mild regularity conditions. 

\begin{assumption}[\textbf{Boundedness and Smoothness}] \label{asm:bound_smooth}
Suppose that  $||\x^{(i)}||\le c_1$, $\left|y^{(i)}\right|\le c_1$ 
for every $i\in[m]$, 
and $\left\Vert \sigma_{+}\right\Vert _{\lip}\le c_1$, $\left\Vert \sigma_{+}\right\Vert _{\infty}\le c_1$ for some $c_1<\infty$.
\end{assumption}

\begin{proposition} \label{thm:convex}
Under Assumption \ref{asm:bound_smooth}, if $S_{n}$ is constructed by (\ref{equ:step}), we have $\DD[f_{S_{n}}]=\mathcal{O}(1/n)+\DD^{*}_N$, for  $\forall n \in[N]$.
\end{proposition}

\paragraph{Remark}
Notice that at iteration $n$, the number of neurons in set $S_n$ is no more than n since in each iteration, we at most increase the number of neurons by one. Also, as we allow select one neuron multiple times the number of neurons in $S_n$ can be smaller than $n$.

Note that the condition of Proposition \ref{thm:convex} is very mild. 
It holds for any original network $f_{[N]}$ of any size,
although it is favorable to make $N$ large to obtain a small $\DD^*_N$. In the sequel, we show that  a faster $\mathcal O(1/n^2)$ rate can be  achieved, if $f_{[N]}$ is sufficiently large and can ``overfit'' the training data in a proper sense. 

\subsection{Faster Rate With Over-parameterized Networks} \label{subsec:faster_N}

We now establish 
the faster rate $\DD[f_{S_n}]=\mathcal{O}(1/n^2)$ when the large network is properly over-paramterized, which 
outperforms the $\mathcal O(1/n)$ rate achieved by standard gradient descent. 
This provides a theoretical foundation for the widely used approach of learning small networks by pruning from large networks. 

Specifically, our result requires that 
$N$ is sufficiently large and 
the neurons in $f_{[N]}$  are independent and diverse enough 
such that 
we can use a convex combination of 
$N$ neurons 
to perfectly fit the data $\mathcal D_m$, 
even when subject to arbitrary perturbations on the labels with bounded magnitude.  

\begin{assumption} [\textbf{Over-parameterization}] \label{asm:inter_N}
There exists a constant $\gamma>0$ such that for any $\boldsymbol{\epsilon}=[\epsilon^{(1)},...,\epsilon^{(m)}]\in\R^m$ with $||\boldsymbol{\epsilon}||\le\gamma$, there exists $[\alpha_1,...,\alpha_N]\in\R^N$ (which may depends on $\boldsymbol{\epsilon}$) with $\alpha_i\in[0,1]$ and $\sum_{i=1}^N\alpha_i=1$ such that for all $(\x^{(i)},y^{(i)})$, $i\in[m]$,
\[
\sum_{j=1}^{N}\alpha_{i}\sigma(\th_j,\x^{(i)})= y^{(i)}+\epsilon^{(i)}.
\]
Note that this implies that $\DD^*_{N} = 0$. 
\end{assumption}
This  roughly requires that the original large network should be sufficiently over-parametrized to have more independent neurons than data points to overfit arbitrarily perturbed labels (with a bounded magnitude). 
As we discuss in Appendix~\ref{sec:inter_discuss}, 
Assumption~\ref{asm:inter_N}
can be shown to be equivalent to the interior point condition of Frank-Wolfe algorithm \cite{bach2012equivalence,lacoste2016convergence, chen2012super}. 

\begin{theorem} [\textbf{Faster Rate}] \label{lem:Nnet}
Under assumption \ref{asm:bound_smooth} and \ref{asm:inter_N}, 
for $S_n$ defined in \eqref{equ:step}, 
we have 
\begin{align} \label{equ:2nrate}
\DD[f_{S_n}] = \mathcal{O}(1/(  \min(1,\gamma) n)^{2}).
\end{align}
\end{theorem}

\subsection{Assumption~\ref{asm:inter_N} Under Gradient Descent} \label{subsec:faster_mf}

In this subsection, we show that Assumption \ref{asm:inter_N} holds with high probability when $N$ is sufficiently large and 
the large network $f_{[N]}$ is trained using gradient descent with a proper random initialization. 
 Our analysis builds on the mean field analysis of neural networks  \cite{song2018mean, mei2019meandimfree}. We introduce the background before we proceed. 

\paragraph{Gradient Dynamics}
 Assume the parameters $\{\vv\theta_i\}_{i=1}^N$ of $f_{[N]}$ are trained using a continuous-time gradient descent (which can be viewed as gradient descent with infinitesimal step size), with a random initialization:
\begin{align} \label{equ:gd}
\frac{d}{dt} \vv \vartheta_i(t) = \vv g_i(\vv \vartheta(t)),  && 
\vv\vartheta_i(0) \overset{{\text{i.i.d.}}}{\sim} \rho_0, && 
\forall i \in [N], 
\end{align}
where $\vv g_i(\vv\vartheta)$ denotes the negative gradient of loss w.r.t.  $\vv\vartheta_i$, 
$$
\vv g_i(\vv\vartheta(t)) = 
\E_{(\vv x,y)\sim \mathcal D_m} [(y - f(\x;~\vv\vartheta(t)) \nabla_{\vv\vartheta_i}\sigma(\x,\vv\vartheta_i(t))],
$$
and $f(\x; ~\vv\vartheta) = \sum_{i=1}^N \sigma(\x,\vv\vartheta_i)/N$. 
Here we initialize $\vv\vartheta_i(0)$ by drawing i.i.d. samples from some distribution $\rho_0$.

\begin{assumption} \label{asm:init}
Assume $\rho_0$ is an absolute continuous distribution on $\R^d$ with a bounded support. Assume the parameters $\{\vv \theta_i\}$ in $f_{[N]}$ are obtained by running \eqref{equ:gd} for some finite time $T$, that is, $\vv\theta_i = \vv\vartheta_i(T)$, $\forall i\in[N]$.
\end{assumption}
 
\paragraph{Mean Field Limit}
We can represent a neural network using the empirical distribution of the parameters. 
Let $\rho_t^N$ be the empirical measure of $\{\vv\vartheta_i(t)\}_{i=1}^N$ at time $t$, i.e., 
$
\rho_{t}^N :=  \sum_{i=1}^N\delta_{\vv\vartheta_i(t)}/N
$
where $\delta_{\vv\vartheta_i}$ is Dirac measure at $\vv\vartheta_i$. We can represent the network $f(\x;\vv\vartheta(t))$ by $f_{\rho_t^N}:=\E_{\vv\vartheta\sim\rho_t^N}\left [\sigma(\vv\vartheta,\x)\right]$. Also,  $f_{[N]}=f_{\rho_T^N}$ under Assumption \ref{asm:init}.

The mean field analysis amounts to study the limit behavior of the neural network with an infinite number of neurons. 
Specifically, as $N\to \infty$, 
it can be shown that $\rho_t^N$ weakly converges to a limit distribution  $\rho_t^\infty$, and $f_{\rho_t^\infty}$ can be viewed as the network with infinite number of neurons at training time $t$. It is shown that $\rho_t^\infty$ is characterized by a partial differential equation (PDE)  
\citep{song2018mean, mei2019meandimfree}: 
 \begin{align} \label{equ:rhoinfty}
 \frac{d}{dt}  \rho_t^\infty = \nabla \cdot (\rho_t^\infty \vv g[\rho_t^\infty]),&&
 \rho_0^\infty = \rho_0, 
 \end{align}
 where $\vv g[\rho_t^\infty](\vv\vartheta) = \E_{(\x,y)\sim \mathcal D_m}[(y-f_\rho(\x)) \nabla_{\vv\vartheta}\sigma(\x, \vv \vartheta)]$, $f_\rho(\x) = \E_{\vv \vartheta\sim \rho}[\sigma(\x; ~\vv \vartheta)]$, 
  and $\nabla \cdot \vv g$ is the divergence operator. 

The mean field theory needs the following smoothness condition on activation to make sure the PDE (\ref{equ:rhoinfty}) is well defined \citep{song2018mean,mei2019meandimfree}.

\begin{assumption} \label{asm:smooth_2}
The derivative of activation function is Lipschitz continuous, i.e., $||\sigma_+'||_\lip < \infty$.
\end{assumption}

It is noticeable that Assumption \ref{asm:smooth_2} does not hold for ReLU. However, as shown in \citet{song2018mean}, empirically, ReLU networks behave very similarly to networks with smooth activation.

A key result of the mean field theory  
says that 
$\rho_{T}^N$ weakly converges to $\rho_T^\infty$ when $N\to \infty$. It implies that $\DD[f_{\rho_T^n}] = \mathcal O(1/n + \varepsilon)$, with $\varepsilon = \DD[f_{\rho_{T}^\infty}]$. 
As shown in \citet{song2018mean}, $\DD[f_{\rho_{T}^{\infty}}]$ is usually a small term giving that the training time $T$ is sufficiently large, under some regularity conditions.  

\paragraph{Over-parameterization of Mean Field Limit}
The key idea of our analysis is: if the infinitely wide 
limit network $f_{\rho_T^\infty}$  can overfit any noise with bounded magnitude (as defined in Assumption \ref{asm:inter_N}), then Assumption \ref{asm:inter_N} holds  for $f_{\rho_T^N}$ with high probability  when $N$ is sufficiently large. 

\begin{assumption} \label{asm:inter}
There exists $\gamma^* >0$, 
such that for any noise vector $\vv\epsilon =[\epsilon_i]_{i=1}^m$ with $\norm{\vv\epsilon}\leq \gamma^*$, there exists a positive integer  $M$, and  $[\alpha_1,...,\alpha_M]\in\R^M$ with $\alpha_j\in[0,1]$ and $\sum_{j=1}^M\alpha_j=1$ and $\bar{\th}_j\in\supp(\rho_T^\infty),j\in[M]$ such that
\[\sum_{j=1}^{M}\alpha_{j}\sigma(\bar{\th}_{j},\x^{(i)})=y^{(i)}+\epsilon^{(i)},
\]
holds for any $i\in[m]$. Here $M$, $\{\alpha_j, \bar{\th}_j\}$ may depend on $\boldsymbol{\epsilon}$.
\end{assumption}

Assumption~\ref{asm:inter} can be viewed as an infinite variant of Assumption~\ref{asm:inter_N}. It is very mild because $\supp(\rho_T^\infty)$ contains infinitely many neurons and given any $\boldsymbol{\epsilon}$, we can pick an arbitrarily large number of neurons from $\supp(\rho_T^\infty)$ and reweight them to fit the perturbed data. Also, assumption~\ref{asm:inter} implicitly requires a sufficient training time $T$ in order to make the limit network 
fit the data well.

\begin{assumption}
[\textbf{Density Regularity}]
\label{asm:density}
For $\forall r_{0}\in(0,\gamma^{*}]$, there exists $p_{0}$ that depends on   $r_0$, such that for every $\th\in\supp(\rho_{T}^{\infty}),$ we have 
$\mathbb{P}_{\th'\sim\rho_{T}^{\infty}}\left(\norm{\th'-\th}\leq r_0 
\right)\ge p_{0}$.
\end{assumption}

\begin{theorem} \label{thm:main}
Suppose Assumption~\ref{asm:bound_smooth},~\ref{asm:init}, \ref{asm:smooth_2}, \ref{asm:inter} and \ref{asm:density} hold, then for any $\delta>0$, when $N$ is sufficiently large, assumption \ref{asm:inter_N} holds for any $\gamma\le\frac{1}{2}\gamma^*$ with probability at least $1-\delta$, which gives that $\DD[f_{S_n}] = \mathcal{O}(1/(\min(1,\gamma) n)^2)$.
\end{theorem}

Theorem \ref{thm:main} shows that if the pre-trained network is sufficiently large, the loss of the pruned network decays at a faster rate. Compared with Proposition \ref{thm:convex}, it highlights the importance of using a large pre-trained network for pruning.

\begin{figure}[t]
    \centering
    \includegraphics[width=0.35\textwidth]{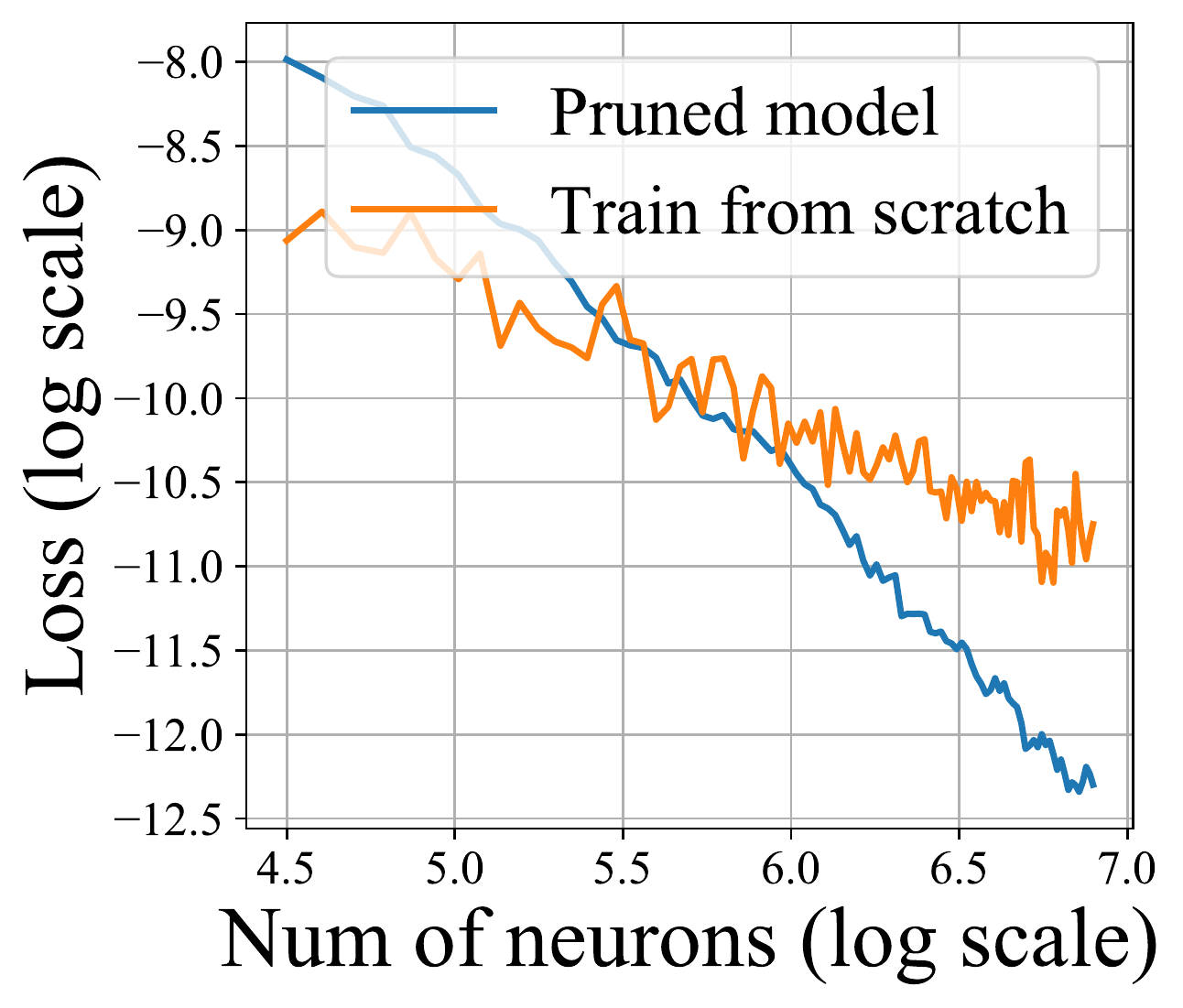}
    \caption{Comparison of loss of the pruned network and train-from-scratch network with varying sizes. Both the loss and number of neurons are in logarithm scale.} \label{fig:toyrate}
\end{figure}

\paragraph{Pruning vs. GD: Numerical Verification of the Rates} 
We numerically verify the fast $\mathcal O(1/n^2)$ rate
in (\ref{equ:2nrate}) and the $\mathcal O(1/n)$ rate of gradient descent by \citet{song2018mean,mei2019meandimfree} (when $\varepsilon$ term is very small). 
Given some simulated data, we first train a large network $f_{[N]}$ with $N=1000$ neurons by gradient descent with random initialization. We then apply our greedy selection algorithm to find subnetworks with different sizes $n$. We also directly train  networks of size $n$ with gradient descent. 
See Appendix~\ref{apx:toy} for more details. 
Figure~\ref{fig:toyrate} plots the the loss $\DD[f]$ and the number of neurons $n$ of the pruned network and the network trained from scratch. This empirical result matches our $\mathcal O(1/n^2)$ rate in Theorem~\ref{thm:main}, and the $\mathcal O(1/n)$ rate of the gradient descent. 

\subsection{Pruning Randomly Weighted Networks}
A line of recent empirical works \citep[e.g.,][]{frankle2018lottery,ramanujan2019s} 
have demonstrated a \emph{strong lottery ticket hypothesis}, which 
shows that it is possible to 
find a subnetwork with good accuracy inside a large  network with random weights without pretraining. 
Our analysis is also applicable to this case. 
Specifically, 
the $\mathcal L[f_{S_n}] = \mathcal O(1/n^2)$ bound in Theorem~\ref{thm:main} holds even when the weights $\{\vv \theta_i\}$ of the large network is i.i.d. drawn from the initial distribution $\rho_0$, without further training.  This is because Theorem~\ref{thm:main} applies to any training time $T$, including $T= 0$ (no training). {See Appendix \ref{apx:random} for a more thorough discussion}.

\subsection{Greedy Backward Elimination}
To better illustrate the advantages of the forward selection approach over backward  elimination (see Figure~\ref{fig:vs}), 
it is useful to consider the 
backward elimination counterpart of our method  which minimizes the same loss as our method, but from the opposite direction. 
That is, it starts from the full network $S_{0}^{\text{B}}:=[N]$, and sequentially  deletes neurons via 
\[
S_{n+1}^{\text{B}}\leftarrow S_{n}^{\text{B}}\setminus \{i_{n}^{*}\},\ \ \ \text{where}\ i_{n}^{*}=\underset{i\in S_{n}^{\text{B}}}{\arg\min}\DD[f_{S_{n}^{\text{B}}\setminus \{i\}}].
\]
As shown in Appendix \ref{apx: gbe}, 
this backward elimination 
does not enjoy similar $\mathcal O(1/n)$ or $\mathcal O(1/n^2)$ rates as forward selection and simple counter examples can be constructed easily.  
Additionally, 
Table ~\ref{tbl:gbe_gfsddfd} in Appendix~\ref{apx: gbe}  
shows that the forward selection outperforms  this backward elimination on both ResNet34 and  MobileNetV2 for  ImageNet.

\subsection{Further Discussion}
To the best of our knowledge, 
our work provides the first rigorous theoretical justification  
that 
pruning from over-parameterized models 
\emph{outperforms} direct training of small networks from scratch using gradient descent, under rather practical assumptions.
However, there still remain gaps between theory and practice that deserve further investigation in future works. 
Firstly, we only analyze the simple two-layer networks but we believe our theory can be generalized to deep networks with refined analysis and more complex theoretical framework such as deep mean field theory \citep{araujo2019mean,nguyen2020rigorous}. We conjecture that pruning deep network gives $\mathcal{O}(1/n^2)$ rate with the constant depending on the Lipschitz constant of the mapping from feature map to output. 
Secondly, as we only analyze the two-layer networks, our theory cannot characterize whether pruning finds good structure of deep networks, as discussed in \citet{liu2018rethinking}. Indeed, theoretical works on how network architecture influences  the performance are  still largely missing. Finally, some of our analysis is built on the mean field theory, which is a special parameterization of network. It is also of interest to generalize our theory to other parameterizations, such as these based on neural tangent kernel \citep{jacot2018neural, du2018gradient_1}.

\begin{algorithm}[t]
\caption{Layer-wise Greedy Subnetwork Selection} \label{alg0}
    \begin{algorithmic}
        \STATE{\textbf{Goal}: Given a pretrained network $f_{\text{Large}}$ with $H$ layers, find a subnetwork $f$ with high accuracy.}
        \STATE Set $f  = f_{\text{Large}}$. 
        \FOR{Layer $h \in [H]$ (From input layer to output layer)}
	       \STATE{Set $S = \emptyset$}
	       \WHILE{Convergence criterion is not met}
		      \STATE{Randomly sample a mini-batch data $\hat{\D}$}
		      \FOR{filter (or neuron) $k \in [N_h]$}
			     \STATE{$S'_k\leftarrow S \cup \{k\}$}
			     \STATE{Replace layer $h$  of $f$ by $\sum_{j\in [S'_k]}\sigma(\boldsymbol{\theta}_{j},\z^{\text{in}})/\left|S'_k\right|$}
			     \STATE{Calculate its loss 
			     $\ell_k$ 
			     on mini-batch data $\hat{\D}$.  
			     } 
		      \ENDFOR
		      \STATE{$S\leftarrow S\cup \{k^{*}\}$, where $k^{*}=\underset{k\in[N_{h}]}{\arg\min}\ \ell_{k}$}
	       \ENDWHILE
	       \STATE{Replace layer $h$ of $f$ by  $\sum_{j\in [S]}\sigma(\boldsymbol{\theta}_{j},\z^{\text{in}})/\left|S\right|$}
        \ENDFOR
        \STATE{Finetune the subnetwork $f$.}
    \end{algorithmic}
\end{algorithm}

\section{Practical Algorithm and Experiments} \label{sec:algo}

\paragraph{Practical Algorithm}
We propose to apply the greedy selection strategy in a layer-wise fashion in order to prune neural networks with multiple layers. Assume we have a pretrained deep neural network with $H$ layers, whose $h$-th layer 
contains $N_h$ neurons and defines a mapping $\sum_{j\in[N_{h}]}\sigma(\boldsymbol{\theta}_{j},\z^{\text{in}})/N_{h}$, where $\z^{\text{in}}$ denotes the input of this layer. 
To extend  the greedy subnetwork selection to deep networks, we propose to prune the layers sequentially, from the input layer to the output layer. 
For each layer, 
we first remove all the neurons in that layer, and gradually add the best neuron back that yields the largest decrease of the loss, similar to the updates in \eqref{equ:step}. 
After finding the subnetwork for all the layers, we further finetune the pruned network,
training it with stochastic gradient descent  
using the weight of original network as initialization. 
This allows us to inherit the accuracy and information in the pruned subnetwork, because finetuning can only decrease the loss over the initialization.
We summarize the detailed procedure of our method in Algorithm \ref{alg0}. Code for reproducing can be found at \url{https://github.com/lushleaf/Network-Pruning-Greedy-Forward-Selection}.

\begin{table*}[h]
    \centering
    \begin{tabular}{l|l|c|c|c}
        \hline
        Model & Method & Top-1 Acc & Size (M) & FLOPS \\
        \hline
        \multirow{8}{*}{ResNet34} & Full Model \citep{he2016resnet}  & 73.4 & 21.8 & 3.68G \\
        \cline{2-5}
        & \citet{li2016pruning} & 72.1 & - & 2.79G \\
        & \citet{liu2018rethinking} & 72.9 & - & 2.79G \\
        & \citet{dong2017more} & 73.0 & - & 2.75G \\
        & Ours & \bf{73.5} & \bf{17.2} & \bf{2.64G} \\
        \cline{2-5}
        & SFP~\citep{he2018soft} & 71.8 & - & 2.17G \\
        & FPGM~\citep{he2019filter} & 72.5 & - & 2.16G \\
        & Ours & \bf{72.9} & \bf{14.7} & \bf{2.07G} \\
        \hline
        \multirow{18}{*}{MobileNetV2} & Full Model~\citep{sandler2018mobilenetv2} & 72.0 & 3.5 & 314M\\
        & Ours & \bf{71.9} & \bf{3.2} & \bf{258M} \\
        \cline{2-5}
        & LeGR~\citep{chin2019legr} & 71.4 & - & 224M \\
        & Uniform~\citep{sandler2018mobilenetv2} & 70.4 & 2.9 & 220M \\
        & AMC~\citep{he2018amc} & 70.8 & 2.9 & 220M \\
        & Ours & \bf{71.6} & \bf{2.9} & \bf{220M} \\
        & Meta Pruning \citep{liu2019metapruning} & 71.2 & - & 217M \\
        & Ours & \bf{71.2} & \bf{2.7} & \bf{201M} \\
        \cline{2-5}
        & ThiNet~\citep{luo2017thinet} & 68.6 & - & 175M \\
        & DPL~\citep{zhuang2018dpl} & 68.9 & - & 175M \\
        & Ours & \bf{70.4} & \bf{2.3} & \bf{170M} \\ 
        \cline{2-5}
        & LeGR~\citep{chin2019legr} & 69.4 & - & 160M \\
        & Ours & \bf{69.7} & \bf{2.2} & \bf{152M} \\
        \cline{2-5}
        & Meta Pruning \citep{liu2019metapruning} & 68.2 & - & 140M \\
        & Ours & \bf{68.8} & \bf{2.0} & \bf{138M} \\
        \cline{2-5}
        & Uniform~\citep{sandler2018mobilenetv2} & 65.4 & - & 106M \\
        & Meta Pruning \citep{liu2019metapruning} & 65.0 & - & 105M \\
        & Ours & \bf{66.9} & \bf{1.9} & \bf{107M} \\
        \hline
        \multirow{3}{*}{MobileNetV3-Small} &
        Full Model~\citep{howard2019searching} & 67.5 & 2.5 & 64M \\
        \cline{2-5}
        & Uniform~\citep{howard2019searching} & 65.4 & 2.0 & 47M \\
        & Ours & \bf{65.8} & 2.0 & \bf{49M} \\
        \hline
        \multirow{3}{*}{ProxylessNet-Mobile} & Full Model~\citep{cai2018proxylessnas} & 74.6 & 4.1 & 324M \\
        \cline{2-5}
        & Uniform~\citep{cai2018proxylessnas} & 72.9 & 3.6 & 240M \\
        & Ours & \bf{74.0} & 3.4 & \bf{232M} \\
        \hline
    \end{tabular}
    \caption{Top1 accuracies for different benchmark models, e.g.  ResNets~\citep{he2016resnet}, MobileNetV2~\citep{sandler2018mobilenetv2}, MobileNetV3-small \citep{howard2019searching} and ProxylessNet \citep{cai2018proxylessnas} on ImageNet2012~\citep{deng2009imagenet}.}
    \label{tab:imagenet}
    \vspace{-0.5cm}
\end{table*}

\paragraph{Empirical Results}
We first apply the proposed algorithm to prune various models, e.g. ResNet~\citep{he2016resnet}, MobileNetV2~\citep{sandler2018mobilenetv2}, MobileNetV3~\citep{howard2019searching} and ProxylessNet~\citep{cai2018proxylessnas} for ImageNet~\citep{deng2009imagenet} classification.
We also show the experimental results on CIFAR-10/100 in the appendix.
Our results are summarized as follows: 

i) Our greedy selection method  
consistently outperforms the prior arts on network pruning on learning small and accurate networks with high computational efficiency. 

ii) \!
Finetuning pruned subnetworks of neural architectures  
  (e.g., MobileNetV2/V3) consistently  outperforms re-training them from new random initialization,  
violating the results of \citet{liu2018rethinking}.

iii) 
Increasing the size of the pre-trained large networks 
 improves the performance of the pruned subnetworks, 
 highlighting the importance of pruning from \emph{large} models.

\subsection{Finding Subnetworks on ImageNet}

We use ILSVRC2012, a subset of ImageNet~\citep{deng2009imagenet} 
which consists of about 1.28 million training images 
and 50,000 validation images with 1,000 different classes.

\begin{figure}[t]
\begin{center}
\renewcommand{\tabcolsep}{2pt}
\begin{tabular}{c}

\raisebox{2.25em}{\rotatebox{90}{{\large Test Accuracy}}}~~
\includegraphics[width =0.45\textwidth]{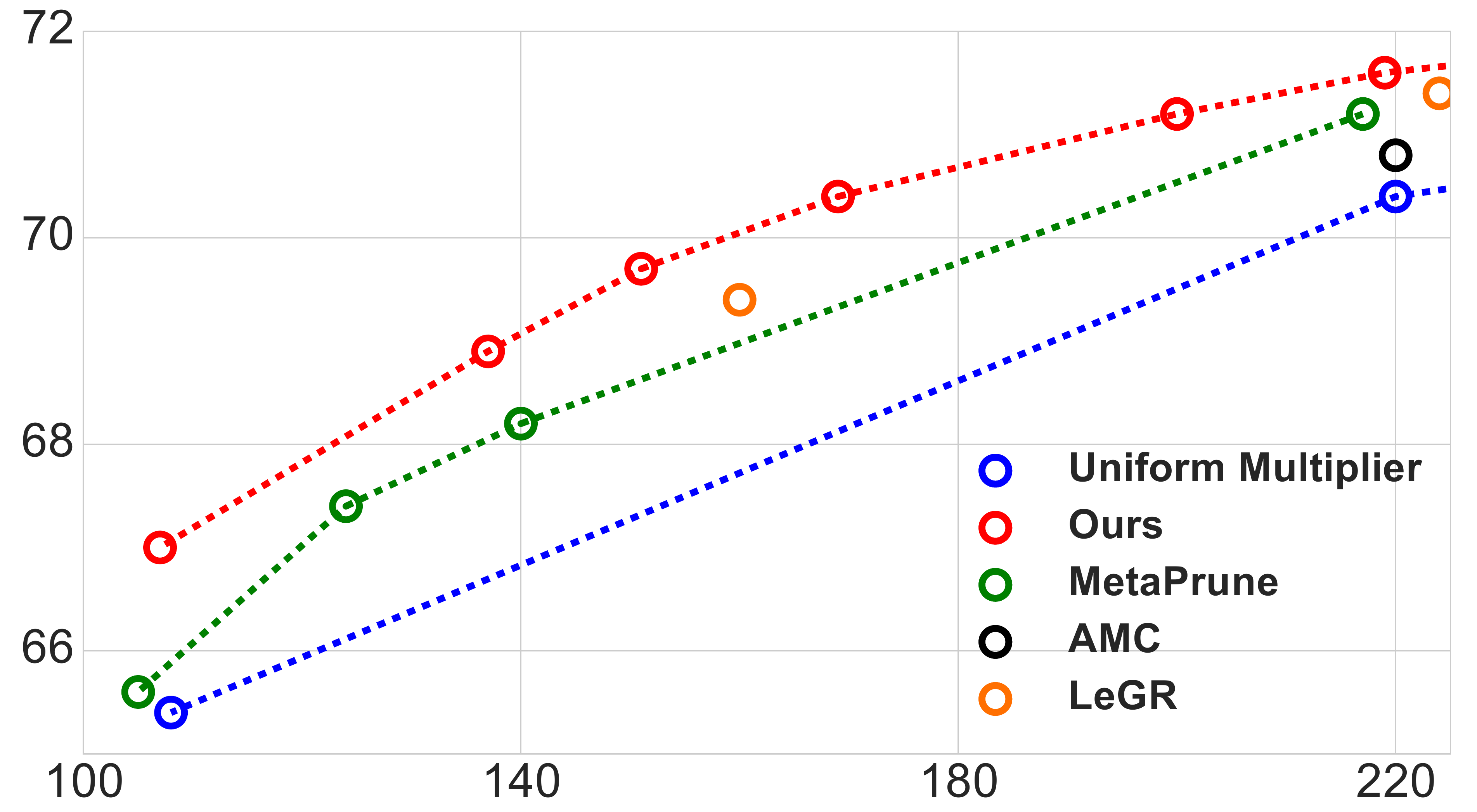}\\
~~~~FLOPs (M) \\
\end{tabular}
\end{center}
\caption{\small 
After applying different pruning algorithms to MobileNetV2  on ImageNet, we display the top1 accuracy of different methods.
It is obvious that our algorithm can consistently outperform all the others under any FLOPs.
}
\label{fig:imagenet}
\end{figure}

\paragraph{Training Details}
We evaluate each neuron using a mini-batch of training data to select the next one to add, as shown in Algorithm \ref{alg0}. 
We stop adding new neurons when the gap between the current loss and the loss of the original pre-trained model is smaller than $\epsilon$. 
We vary $\epsilon$ to get pruned models with different sizes.

During finetuning, 
we use the standard SGD optimizer with Nesterov momentum 0.9 and
weight decay $5 \times 10^{-5}$. 
For ResNet, we use a fixed learning rate $2.5\times10^{-4}$.
For the other architectures,
following the original settings~\citep{cai2018proxylessnas, sandler2018mobilenetv2},
we decay learning rate using cosine schedule ~\citep{loshchilov2016cosine} starting from 0.01. We finetune subnetwork for 150 epochs with batch size 512 on 4 GPUs. We resize images to $224 \times 224$ resolution and adopt the standard data augmentation scheme (mirroring and shifting). 

\paragraph{Results}
Table ~\ref{tab:imagenet} reports the top1 accuracy, FLOPs and model size \footnote{ All the FLOPS and model size reported in this paper is calculated by \url{https://pypi.org/project/ptflops}.} of subnetworks pruned from the full networks.
We first test our algorithm on two standard benchmark models, 
ResNet-34 and MobileNetV2.
We further apply our algorithm to several recent proposed models e.g., ProxylessNet, MobileNetV3-Small.

\paragraph{ResNet-34}
Our algorithm outperforms all the prior results on ResNet-34.
We obtain an even better top1 accuracy (73.4\% vs. 73.5\%) than the full-size network
while reducing the FLOPs from 3.68G to 2.64G.
We also obtain a model with $72.9\%$ top1 accuracy and $2.07$G FLOPs, which has higher accuracy but lower FLOPs than previous works.

\paragraph{MobileNetV2}
Different from ResNet and other standard structures, 
MobileNetV2 on ImageNet is known to be hard to prune using most traditional pruning algorithms~\citep{chin2019legr}.
As shown in Table~\ref{tab:imagenet}, 
compared with the `uniform baseline', which uniformly reduces the number of channels on each layer, most popular algorithms fail to improve the performance by a large margin.
In comparison, our algorithm improves the performance of small-size networks by a significant margin.
As shown in Table ~\ref{tab:imagenet},
the subnetwork with 245M FLOPs obtains 71.9\% top1 accuracy, 
which matches closely with the 72.0\% accuracy of the full-size network.
Our subnetwork with 151M FLOPs achieves 69.7\% top1 accuracy, 
improving the previous state-of-the-art of 69.4\% top1 accuracy with 160M FLOPs.
As shown in Figure~\ref{fig:imagenet},   
our algorithm consistently outperforms all the other baselines under all FLOPs.
The improvement of our method on 
the low FLOPs region is particularly significantly. 
For example, when 
limited to 106M FLOPs, 
we improve the 65.0\% top1 accuracy of Meta Pruning to 66.9\%.

\paragraph{ProxylessNet-Mobile and MobileNetV3-Small}
We further experiment on two recently-proposed architectures, 
ProxylessNet-Mobile and MobileNetV3-Small.
As shown in Table~\ref{tab:imagenet}, we consistently outperform the `uniform baseline'.
For MobileNetV3-Small, we improve the $65.4\%$ top1 accuracy to $65.8\%$ when the FLOPs is less than 50M FLOPs. For ProxylessNet-Mobile, we enhance the $72.9\%$ top1 accuracy to $74.0\%$ when the FLOPs is under 240M.

\subsection{Rethinking the Value of Finetuning}
\label{sec:rethinking}
Recently, ~\citet{liu2018rethinking} finds that 
for ResNet, VGG 
and other standard structures on ImageNet, 
re-training the weights of the pruned structure
from new random initialization can achieve better performance than finetuning.
However, we find that this claim does not hold for mobile models, 
such as MobileNetV2 and MobileNetV3.
In our experiments, we use the same setting of \citet{liu2018rethinking}
for re-training from random initialization.

\begin{table}[htbp!]
    \centering
    \begin{tabular}{l|l|c|c}
        \hline
        Models & FLOPs & Re-training (\%) & Finetune (\%) \\
        \hline
        MobileNetV2 & 220M & 70.8 & 71.6 \\
        \hline
        MobileNetV2 & 170M & 69.0 & 70.4 \\
        \hline
        MobileNetV3 & 49M & 63.2 & 65.8 \\
        \hline
    \end{tabular}
    \caption{Top1 accuracy on MobileNetV2 and MobileNetV3-small on ImageNet.
    ``Scratch'' denotes training the pruned model from scratch. We use the Scratch-B setting in \citet{liu2018rethinking} for training from scratch. }
    \label{tab:rethink}
\end{table}

We compare finetuning and re-training of the pruned MobileNetV2 with 219M and 169M Flops.
As shown in Table~\ref{tab:rethink}, finetuning outperforms 
re-training by a large margin.
For example, for the 169M FLOPs model, 
re-training decreases the top1 accuracy from 70.4\% to 69.0\%.
This empirical evidence demonstrates the importance of using the weights learned by the large model to initialize the pruned model.

We conjecture that the difference between our findings and \citet{liu2018rethinking} might come from several reasons.
Firstly, for large architecture such as VGG and ResNet, the pruned model is still large enough (e.g. as shown in Table~\ref{tab:imagenet}, FLOPs $>$ 2G) to be optimized from scratch.
However, this does not hold for the pruned mobile models, which is much smaller.
Secondly, \citet{liu2018rethinking} 
mainly focus on sparse regularization based pruning methods such as \citet{han2015learning, li2016pruning, he2017channel}. In those methods, the loss used for training the large network has an extra strong regularization term, e.g., channel-wise $L_p$ penalty. However, when re-training the pruned small network, the penalty is excluded. This gives inconsistent loss functions. As a consequence, the weights of the large pre-trained network may not  be suitable for finetuning the pruned model. In comparison, our method uses the same loss for training the re-trained large model and the pruned small network, both without regularization term. A more comprehensive understanding of this issue is valuable to the community, which we leave as a future work.

However, we believe that
a more comprehensive understanding of finetuning is valuable to the community, which we leave as a future work.

\begin{table}[htbp!]
    \centering
    \begin{tabular}{l|ccc}
        \hline
        & \multicolumn{3}{c}{Large $N \longrightarrow$ Small $N$} \\
        \hline
        Original FLOPs (M) & 320 & 220 & 108 \\
        \hline
        Pruned FLOPs (M) & 96 & 96 & 97 \\
        \hline
        Top1 Accuracy (\%) & 66.2 & 65.6 & 64.9 \\
        \hline
    \end{tabular}
    \caption{We apply our algorithm to get three pruned models with similar FLOPs from full-size MobileNetV2, MobileNetV2$\times$0.75 and MobileNetV2$\times$0.5.}
    \label{tab:LargeVsSmall}
    \vspace{-0.5cm}
\end{table}

\subsection{On the Value of Pruning from Large Networks}
\label{sec:large}
Our theory suggests it is better to prune from a larger model, as discussed in Section~\ref{sec:theory}.
To verify, we apply our method to MobileNetV2 with different sizes, including MobileNetV2 (full size), MobileNetV2$\times$0.75 and MobileNetV2$\times$0.5 \citep{sandler2018mobilenetv2}. We keep the FLOPs of the pruned models almost the same and compare their performance. As shown in Table~\ref{tab:LargeVsSmall}, the pruned models from larger original models give better performance. For example, the 96M FLOPs pruned model from the full-size MobileNetV2 obtains a top1 accuracy of $66.2\%$ while the one pruned from MobileNetV2$\times$0.5 only has $64.9\%$. 

\section{Related Works}
\textbf{Structured Pruning}
A vast literature exists on  
\emph{structured} pruning 
\citep[e.g.,][]{han2016eie}, which prunes neurons, channels or other units of neural networks.  
Compared with 
weight pruning~\citep[e.g.,][]{han2015compress}, 
which specifies the connectivity of neural networks, 
structured pruning 
is more realistic as 
it can compress neural networks without dedicated hardware or libraries.
Existing methods 
prune the redundant neurons based on different criterion, 
including the norm of the weights 
\citep[e.g.,][]{liu2017slim, zhuang2018dpl, li2016pruning}, 
 feature reconstruction error of the next or final layers
\citep[e.g.,][]{he2017channel, yu2018nisp, luo2017thinet},
or gradient-based sensitivity measures 
\citep[e.g.,][]{baykal2019sipping, zhuang2018dpl}.  
Our method is designed to directly minimize the final loss, and yields both better practical performance and theoretical guarantees. 

\textbf{Forward Selection vs. Backward Elimination}
Many of the popular conventional network pruning methods are based on 
backward elimination of redundant neurons \citep[e.g.][]{liu2017slim, li2016pruning, yu2018nisp}, 
and fewer algorithms are based forward selection like our method \citep[e.g.][]{zhuang2018dpl}.
Among the few exceptions, \citet{zhuang2018dpl} propose a greedy channel selection algorithm similar to ours, 
but their method is based on minimizing a gradient-norm based sensitivity measure (instead of the actual loss like us), and yield no theoretical guarantees. Appendix~\ref{apx: gbe} discusses the theoretical and empirical advantages of forward selection over backward elimination.

\textbf{Sampling-based Pruning}
Recently, 
a number of works  
\citep{baykal2018data, liebenwein2019provable, baykal2019sipping,mussaydata2020} 
proposed to prune networks based  on variants of (iterative) random sampling according to certain sensitivity score.  
These methods can provide concentration bounds 
on the difference of output between the pruned networks and the full networks, which may yield a bound of $\mathcal O(1/n+\DD[f_{[N]}])$ with a simple derivation. 
Our method uses a simpler greedy deterministic selection strategy and achieves better rate than random sampling in the  overparameterized cases 
In contrast, sampling-based pruning may not yield the fast $\mathcal O(1/n^2)$ rate even with overparameterized models. 
Unlike our method, 
these works do not justify the advantage of pruning from large models over direct gradient training. 

\textbf{Lottery  Ticket; Re-train After Pruning}
\citet{frankle2018lottery} proposed the Lottery Ticket Hypothesis, claiming the existence of  winning subnetworks inside large models. 
\citet{liu2018rethinking} regards pruning as a kind for neural architecture search.
A key difference between our work and 
\citet{frankle2018lottery}  and \citet{liu2018rethinking}
is how the parameters of the subnetwork are trained: 

i) We finetune the parameters of the subnetworks starting from the weights of the pre-trained large model, hence inheriting the information the large model. 

ii) \citet{liu2018rethinking} proposes to re-train the parameters of the pruned subnetwork starting from new random initialization. 

iii) \citet{frankle2018lottery} proposes to re-train the pruned subnetwork starting from the same initialization and random seed used to train the pre-trained model. 

Obviously, the different parameter training of subnetworks 
should be combined with different network pruning strategies to achieve the best results. Our algorithmic and theoretical framework naturally justifies the finetuning approach. 
Different theoretical frameworks for justifying the proposals of \citet{liu2018rethinking}
and \citet{frankle2018lottery} (equipped with their corresponding subnetwork selection methods) are of great interest. 


More recently,  
a concurrent work 
\citet{malach2020proving} discussed a stronger form of lottery ticket hypothesis that shows    the existence of winning subnetworks in large networks with random weights (without  pre-training), 
which corroborates the empirical observations in \citep{wang2019pruning, ramanujan2019s}.  
However, 
the result of \citet{malach2020proving} 
does not yield fast rate as our framework for justifying the advantage of network pruning over training from scratch, and does not motivate practical algorithms for finding good subnetworks in practice. 

\textbf{Frank-Wolfe Algorithm}
 As suggested in \citet{bach2017breaking},  
 Frank-Wolfe \citep{frank1956algorithm} 
 can be applied to learn neural networks, 
 which yields an algorithm that greedily adds neurons to  progressively construct a network. 
 However, 
 each step of Frank-Wolfe  
 leads to a challenging global optimization  problem, which can not be
 solved in practice. Compared with \citet{bach2017breaking}, 
our subnetwork selection approach can be 
viewed as constraining the 
global optimization 
the discretized search space constructed using over-parameterized large networks pre-trained using gradient descent. 
Because gradient descent on over-parameterized networks is shown to be nearly optimal  \citep[e.g.,][]{song2018mean,mei2019meandimfree,du2018gradient_1,du2018gradient,jacot2018neural}, selecting neurons inside the pre-trained models 
can provide a good approximation to the original non-convex problem. 

\textbf{Sub-modular Optimization}
An alternative general framework for analyzing greedy selection algorithms is based on sub-modular optimization
\citep{nemhauser1978analysis}. 
However, our problem ~\ref{equ:minF} is not sub-modular, 
and the (weak) sub-modular analysis  \citep[][]{das2011submodular}  
can only bound the ratio between $\DD[f_{S_n}]$ and the best loss of subnetworks of size $n$ achieved by \eqref{equ:minF}, not the best loss $\DD_N^*$ achieved by the best convex combination of all the $N$ neurons in the large model.  

\section{Conclusion}

We propose a simple and efficient greedy selection algorithm for constructing subnetworks from pretrained large networks. 
Our theory provably  justifies the advantage of pruning from large models over 
training small networks from scratch.
The importance of using sufficiently large, over-parameterized models and finetuning (instead of re-training) the selected subnetworks are emphasized. 
Empirically, 
our experiments verify our theory and 
show that our method improves the prior arts on pruning various models such as ResNet-34 and MobileNetV2 on Imagenet. 
\clearpage

\bibliographystyle{icml2020}
\bibliography{main}

\clearpage
\onecolumn

\section{Details for the Toy Example} \label{apx:toy}
Suppose we train the network with $n$ neurons for $T$ time using gradient descent with random initialization, i.e., the network we obtain is $f_{\rho_T^n}$ using the terminology in Section \ref{subsec:faster_mf}. As shown by \citet{song2018mean,mei2019meandimfree}, $\DD[f_{\rho_T^n}]$ is actually $\mathcal{O}(1/n+ \epsilon)$ with high probability, where $\epsilon = \DD[f_{\rho_T^\infty}]$ is the loss of the mean field limit network at training time $T$. \citet{song2018mean} shows that $\lim_{T\to\infty}\DD[f_{\rho_{T}^{\infty}}]=0$ under some regularity conditions and this implies that if the training time $T$ is sufficient, $\DD[f_{\rho_T^\infty}]$ is generally a smaller term compared with the $\mathcal{O}(1/n)$ term.

To generate the synthesis data, we first generate a neural network $f_{\text{gen}}(\x)=\frac{1}{1000}\sum_{i=1}^{N}b_i\text{sigmoid}(\boldsymbol{a}_{i}^{\top}\x)$, where $\boldsymbol{a}_{i}$ are i.i.d. sample from a 10 dimensional standard Gaussian distribution and $b_{i}$ are i.i.d. sample from a uniform distribution $\text{Unif}(-5,5)$. The training data $\x$ is also generated from a 10 dimensional standard Gaussian distribution. We choose $f_{\text{gen}}(\x)=y$ as the label of data. Our training data consists of 100 data points. The network we use to fit the data is $f=\frac{1}{n}\sum_{i=1}^{n}b_{i}'\text{tanh}(\boldsymbol{a}_{i}'^{\top}\x)$. We use network with 1000 neurons for pruning and the pruned models will not be finetuned. All networks are trained for same and sufficiently large time to converge.

\section{Finding Sub-Networks on CIFAR-10/100}

In this subsection, 
we display the results of applying our proposed algorithm to various model structures on CIFAR-10 and CIFAR-100.
On CIFAR-10 and CIFAR-100, we apply our algorithm to the networks already pruned by network slimming~\citep{liu2017slim} provided by ~\citet{liu2018rethinking}
and show that we can further compress models which have already pruned by the $L_1$ regularization.
We apply our algorithm on the pretrained models, and finetune the model with the same experimental setting as ImageNet.

As demonstrated in Table~\ref{tab:cifar}, 
our proposed algorithm can further compress a model pruned by ~\citet{liu2018rethinking} without or only with little drop on accuracy.
For example, on the pretrained VGG19 on CIFAR-10, ~\citet{liu2017slim} can prune 30\% channels and get $93.81\% \pm 0.14\%$ accuracy.
Our algorithm can prune 44\% channels of the original VGG19 and get $93.78\% \pm 0.16\%$ accuracy, which is almost the same as the strong baseline number reported by \citet{liu2018rethinking}.


\begin{table*}[htbp!]
    \centering
    \begin{tabular}{l|l|l|c|c}
        \hline
        DataSet & Model & Method & Prune Ratio (\%) & Accuracy (\%) \\
        \hline
        \multirow{6}{*}{CIFAR10} & \multirow{2}{*}{VGG19} & \citet{liu2017slim} & 70 & $93.81 \pm 0.14$ \\
        \cline{3-5} 
        & & Ours & \textbf{56} & $\bm{93.78 \pm 0.16}$ \\
        \cline{2-5} 
        & \multirow{4}{*}{PreResNet-164} & \citet{liu2017slim} & 60 & $94.90 \pm 0.04$ \\
        \cline{3-5} 
        & & Ours & \textbf{51} & $\bm{94.91 \pm 0.06}$ \\
        \cline{3-5} 
        & & \citet{liu2017slim} & 40 & $94.71 \pm 0.21$ \\
        \cline{3-5} 
        & & Ours & \textbf{33} & $\bm{94.68 \pm 0.17}$ \\
        \hline
        \multirow{6}{*}{CIFAR100} & \multirow{2}{*}{VGG19} & \citet{liu2017slim} & 50 & $73.08 \pm 0.22$ \\
        \cline{3-5} 
        & & Ours & \textbf{44} & $\bm{73.05 \pm 0.19}$ \\
        \cline{2-5} 
        & \multirow{4}{*}{PreResNet-164} & \citet{liu2017slim} & 60 & $76.68 \pm 0.35$ \\
        \cline{3-5} 
        & & Ours & \textbf{53} & $\bm{76.63 \pm 0.37}$ \\
        \cline{3-5} 
        & & \citet{liu2017slim} & 40 & $75.73 \pm 0.29$ \\
        \cline{3-5} 
        & & Ours & \textbf{37} & $\bm{75.74 \pm 0.32}$ \\
        \hline
    \end{tabular}
    
    \caption{Accuracy on CIFAR100 and CIFAR10. ``Prune ratio" stands for the total percentage of channels that are pruned in the whole network. 
    We apply our algorithm on the models pruned by ~\citet{liu2017slim} and find that our algorithm can further prune the models.
    The performance of \citet{liu2017slim} is reported by \citet{liu2018rethinking}.
    Our reported numbers are averaged by five runs.}
    \label{tab:cifar}
\end{table*}

\section{Discussion on Assumption \ref{asm:inter_N} and \ref{asm:inter}} \label{sec:inter_discuss}
Let $\phi_{j}(\th)=\sigma(\x^{(j)}, \th)/\sqrt{m}$ and 
$\bp(\th)=\left[\phi_{1}(\th),...,\phi_{m}(\th)\right]$ to be the vector of the  outputs of the neuron $\sigma(\x; \th)$ 
scaled by $1/\sqrt{m}$, 
realized on a dataset  $\dataset := \{\x^{(j)}\}_{j=1}^m$. 
We call $\bp(\th)$ the feature map of $\th$. 
Given a large network $f_{[N]}(x) = \sum_{i=1}^N \sigma(\x; \th_i)/N$, define the marginal polytope of the feature map to be 
\[
\M_{N}:=\text{conv}\left\{ \bp(\th_i)\mid i \in \{1,\ldots, N\} \right\},
\]
where $\text{conv}$ denotes the convex hull. Then it is easy to see that Assumption~\ref{asm:inter_N} is equivalent to saying that 
$\y:=[y^{(1)},\ldots, y^{(m)}]/\sqrt{m}$ is in the interior of the marginal polytope $\M_N$, i.e., there exists $\gamma>0$ such that
$\mathcal{B}\left(\y,\gamma\right) \subseteq \M_N$. Here we denote by  $\mathcal{B}\left(\boldsymbol{\mu},r\right)$ the ball with
radius $r$ centered at $\boldsymbol{\mu}$. Similar to Assumption~\ref{asm:inter_N}, 
Assumption~\ref{asm:inter} is equivalent to require that 
$\mathcal{B}\left(\y,\gamma^{*}\right) \subseteq\M$, where  
$$
\M\coloneqq\text{conv}\left\{ \bp(\th)\mid\th\in\supp(\rho_{T}^{\infty})\right \}.$$ 
We may further relax the assumption to assuming $\y$ is in the relative interior (instead of interior) of $\M_N$ and $\M$. However, this requires some refined analysis and we leave this as future work.

It is worth mention that when $\M$ has dimension $m$ and $f_{\rho_T^\infty}$ gives zero training loss, then assumption \ref{asm:inter} holds. Similarly, if $\M_N$ has dimension $m$ and $f_{\rho_T^N}$ gives zero training loss, then assumption \ref{asm:inter_N} holds.

\section{Pruning Randomly Weighted Networks} \label{apx:random}
Our theoretical analysis is also applicable for pruning randomly weighted networks. Here we give the following corollary.

\begin{corollary}
Under Assumption \ref{asm:bound_smooth} and suppose that the  weights $\{\vv\theta_i\}$ of the large  neurons $f_{[N]}(\vv x)$  are i.i.d. drawn from an absolutely continuous  distribution $\rho_{0}$ with a bounded support in $\mathbb{R}^{d}$, without further gradient descent training. 
Suppose that Assumption \ref{asm:inter} and \ref{asm:density} hold for $\rho_{0}$ (changing $\rho_{T}^{\infty}$ to $\rho_{0}$). 
Let $S_{n}^{\text{Random}}$ be the subset obtained by the proposed greedy forward selection \eqref{equ:step} on such $f_{[N]}$ at the $n$-th step. For any $\delta>0$ and $\gamma<\gamma^*/2$, when $N$ is sufficiently large, with probability at least $1-\delta$, we have 
\[
\mathcal{L}[f_{S_{n}^{\text{Random}}}]=\mathcal{O}\left(1/\left(\min\left(1,\gamma\right)n\right)^{2}\right).
\]
\end{corollary}
This corollary is a special case of Theorem \ref{thm:main} when taking the training time to be zero ($T=0$). And as the network is not trained, Assumption \ref{asm:smooth_2} are not needed  for this corollary.

\section{Forward Selection is Better Than Backward Elimination} 
\label{apx: gbe}
A greedy backward elimination can be developed analogous to  
our greedy forward selection,  
in which we start with the full network and greedily eliminate neurons that gives the smallest increase in loss. Specifically, starting from $S_{0}^{\text{B}}=[N]$, we sequentially delete neurons via 
\begin{align} \label{equ:backgood}
S_{n+1}^{\text{B}}\leftarrow S_{n}^{\text{B}}\setminus \{i_{n}\}^{*},\ \ \ \text{where}\ i_{n}^{*}=\underset{i\in S_{n}^{\text{B}}}{\arg\min}\DD[f_{S_{n}^{\text{B}}\setminus \{i\}}],
\end{align}
where $\setminus$ denotes set minus. 
In this section, we demonstrate that the forward selection 
has significant advantages over backward elimination,
from both theoretical and empirical perspectives. 


\paragraph{Theoretical Comparison of Forward and Backward Methods}
Although greedy forward selection guarantees 
$\mathcal O(1/n)$ or $\mathcal O(1/n^2)$ error rate as we show in the paper, backward elimination does not enjoy similar theoretical guarantees.  
This is because the ``effective search space'' of backward elimination 
is more limited than that of forward selection, 
and gradually shrinkage over time. 
Specifically, 
at each iteration of backward elimination \eqref{equ:backgood}, 
the best neuron is chosen among $S_{n}^{\text{B}}$, which shrinks as more  neurons are pruned. In contrast, the new neurons in 
 greedy selection \eqref{equ:step} are always selected from the full set $[N]$, which permits each neuron to be selected at every iteration, for multiple times.
 We now elaborate the theoretical advantages of forward selection vs. backward elimination from 1) the best achievable loss by both methods and 2) the decrease of loss across iterations. 
 
 \emph{$\bullet$ On the lower bound.}  
 In greedy forward selection, one neuron can be selected for multiple times at different iterations, 
 while in backward elimination one neuron can only be deleted once.  
 As a result, the best possible loss achievable by backward elimination is worse than that of greedy elimination. 
 Specifically, because backward elimination yields a subnetwork in which each neuron appears at most once. We have an immediate lower bound of 
 $$\mathcal L [S_{n}^\text{B}] \geq \mathcal L_{N}^{\text{B}*},~~~\forall n\in [N],$$ 
 where 
 $$
\DD_{N}^{\text{B}*}=\min_{\boldsymbol{\alpha}}\left\{ \mathcal{L}[f_{\boldsymbol{\alpha}}]:~~ \alpha_i = \bar \alpha_i /\sum_{i=1}^N {\bar \alpha_i}, ~~~ \bar \alpha_i \in \{0,1\} \right \}. 
 $$
 In comparison, for $S_{n}^{*}$ from forward selection \eqref{equ:step}, we have from Theorem~\ref{thm:convex} that 
 $$\mathcal L[S_{n}^{*}]  = \mathcal O(1/n) + \mathcal L_N^*,
 $$
 where $\mathcal L_N^*$ equals (from Eq~\ref{equ:LN*}) 
 $$
\DD_{N}^{*}=\min_{\boldsymbol{\alpha}}\left\{ \mathcal{L}[f_{\boldsymbol{\alpha}}]:~~ a_{i} \geq 0, ~~~\sum_{i=1}^N \alpha_i = 1\right\}. 
 $$
 This yields a simple comparison result of 
  $$
 \mathcal L[S_{n}^\text{B}] \geq \mathcal L[S_n^*] + (\DD_{N}^{\text{B}*} - \DD_{N}^{*}) +  \mathcal O(1/n). 
 $$
 Obviously, we have $\DD_{N}^{\text{B}*} \geq \DD_{N}^{*}$ because $\DD_{N}^{*}$ optimizes on a much larger set of $\vv\alpha$, 
 indicating that backward elimination is inferior to forward selection.  
 In fact, because $\DD_{N}^{\text{B}*}$ is most likely to be strictly larger than  $\DD_{N}^{*}$ in practice,  
 we can conclude that $\DD[S_n^{\text{B}}] = \Omega(1) + \mathcal L_N^*$ where $\Omega $ is the Big Omega notation. 
 This shows that 
 it is impossible to prove bounds similar to  
 $\mathcal L[S_{n}^{*}]  = \mathcal O(1/n) + \mathcal L_N^*$ in  Theorem~\ref{thm:convex} for backward elimination. 


 

\emph{$\bullet$ On the loss descend.}  
%
The key ingredient for proving the $\mathcal{O}(n^{-1})$ convergence of greedy forward selection 
is a recursive inequality that bounds $\DD[f_{S_{n}}]$ at iteration $n$ using $\DD[f_{S_{n-1}}]$ from the previous iteration $n-1$. Specifically, we have 
\begin{align}\label{equ:decreasefw}
\DD[f_{S_{n}} ]\le \DD_N^* ~+~   \frac{ \DD_N^* -   \DD[f_{S_{n-1}}] }{n} ~+ ~ \frac{C}{n^{2}}, 
\end{align}
where 
$C=\max_{\boldsymbol{u},\boldsymbol{v}}\left\{ \left\Vert \boldsymbol{u}-\boldsymbol{v}\right\Vert^2 :\boldsymbol{u},\boldsymbol{v}\in\mathcal{M}_{N}\right\} $; see Appendix \ref{apx: pf1} for details. And inequality (\ref{equ:decreasefw}) directly implies that 
\[
\DD[f_{S_{n}}]\le\DD_{N}^{*} ~+~\frac{\DD[f_{S_{0}}]-\DD_{N}^{*}}{n},~~~~\forall n \in [N]. 
\]

An importance reason for  this inequality to hold is that the best 
neuron to add is selected from the whole set $[N]$ at each iteration.  
However,  similar result does not hold for 
backward elimination, because the neuron to eliminate is 
selected from $S_{n}^{\text{B}}$, whose  size shrinks when $n$ grows. 
In fact, 
for backward elimination, we guarantee to find counter examples that  
violate a counterpart of \eqref{equ:decreasefw}, as shown in the following result, and thus fail to give the $\mathcal{O}(n^{-1})$ convergence rate.
%
\begin{theorem} 
For the $S_n^{\text{B}}$ constructed 
by backward elimination in \eqref{equ:backgood}.  
There 
exists a full network $f_{[N]}(\vx) = \sum_{i=1}^N\sigma(\vx; ~\vv\theta_i)/N$ 
and a dataset $\mathcal {D}_m=(\vx^{(i)}, y^{(i)})_{i=1}^m$ that satisfies Assumption \ref{asm:bound_smooth}, \ref{asm:inter_N}, such that $\mathcal L_N^{\text{B*}} > 0$ and $\exists n \in [N]$
$$
\DD[f_{S_{N-n}^{\text{B}}}]> \DD_{N}^{\text{B}*} ~+~ \frac{\DD[f_{S_{N}^{\text{B}}}]-\DD_{N}^{\text{B}*}}{n},
$$
In comparison,  
the $S_{n}$ from greedy forward selection satisfies 
\begin{align} \label{equ:gooddfdfd}
\DD[f_{S_{n}}]\le\DD_{N}^{*} ~+~\frac{\DD[f_{S_{0}}]-\DD_{N}^{*}}{n},~~~~\forall n \in [N]. 
\end{align}
In fact, 
on the same instance, we have $\DD_{N}^{*}=0$, and  
the faster rate $\DD[f_{S_{n}}]\le\DD_{N}^{*} = \mathcal{O}(n^{-2})$ 
also holds for greedy forward selection.
\end{theorem} 
\begin{proof} 
Suppose the data set contains 2 data points and we represent the neurons as the feature map as in section \ref{sec:inter_discuss}.
Suppose that $N=43$, $\bp(\th_{1})=[0,1.5]$, $\bp(\th_{2})=[0,0]$,
$\bp(\th_{3})=[-0.5,1]$, $\bp(\th_{4})=[2,1]$ and $\bp(\th_{i})=[(-1.001)^{i-3}+2,1]$,
$i\in\{5,6,....,43\}$ and the target $\textbf{y}=[0,1]$ (it is easy to construct the actual weights of neurons and data points such that the above feature maps hold). Deploying
greedy backward elimination on this case gives that 
\[
\DD[f_{S_{N-n}^{\text{B}}}]>\frac{\DD[f_{S_{N}^{\text{B}}}]-\DD_{N}^{\text{B}*}}{n}+\DD_{N}^{\text{B}*},
\]
for $n\in[38]$, where  $\DD_{N}^{\text{B}*}=\min_{n\in[N]}\DD_{N,n}^{\text{B}*}>0.03$.
In comparison, for greedy forward selection, 
(\ref{equ:gooddfdfd}) holds from the proof of Theorem~\ref{thm:convex}. 
In addition, on the same instance, we can verify that $\DD_{N}^{*}=0$, and 
the faster  $\mathcal{O}(n^{-2})$ convergence rate also holds for greedy forward selection. In deed,
the greedy forward selection is able to achieve 0 loss using two neurons
(by selecting $\bp(\th_{3})$ for four times and $\bp(\th_{4})$ once). 
\end{proof} 

\paragraph{Empirical Comparison of Forward and Backward Methods}
We compare 
forward selection and  backward elimination to prune Resnet34 and MobilenetV2 on Imagenet. As shown in Table \ref{tbl:gbe_gfsddfd}, 
forward selection tends to achieve better top-1 accuracy in all the cases, 
which is consistent with the theoretical analysis above.  
The experimental settings of the greedy backward elimination is the same as that of the greedy forward selection. 

\begin{table}
\begin{centering}
\begin{tabular}{c|c|c|c}
\hline 
Model & Method & Top1 Acc & FLOPs\tabularnewline
\hline 
\multirow{4}{*}{ResNet34} & Backward & 73.1 & 2.81G\tabularnewline
 & Forward & 73.5 & 2.64G\tabularnewline
\cline{2-4} \cline{3-4} \cline{4-4} 
 & Backward & 72.4 & 2.22G\tabularnewline
 & Forward & 72.9 & 2.07G\tabularnewline
\hline 
\multirow{4}{*}{MobileNetV2} & Backward & 71.4 & 257M\tabularnewline
 & Forward & 71.9 & 258M\tabularnewline
\cline{2-4} \cline{3-4} \cline{4-4} 
 & Backward & 70.8 & 215M\tabularnewline
 & Forward & 71.2 & 201M\tabularnewline
\hline 
\end{tabular}\caption{Comparing greedy forward selection and backward elimination on Imagenet.}
\label{tbl:gbe_gfsddfd}
\par\end{centering}
\end{table}

\section{Proofs}
Our proofs use the definition of the convex hulls defined in Section \ref{sec:inter_discuss} of Appendix.

\subsection{Proof of Proposition \ref{thm:convex}} \label{apx: pf1}
The proof of Proposition \ref{thm:convex} follows the standard argument of proving the convergence rate of Frank-Wolfe algorithm with some additional arguments. Our algorithm is not a Frank-Wolfe algorithm, but as illustrated in the subsequent proof, 
we can essentially use the Frank-Wolfe updates to control the error of our algorithm. 

Define $\ell(\u)=\left\Vert \u-\y\right\Vert ^{2}$, 
then the subnetwork selection problem can be viewed as solving 
$$
\min_{\u \in \M_N} \ell(\u), 
$$
with $\mathcal L_N^* = \min_{\u \in \M_N} \ell(\u)$. 
And our algorithm 
can be viewed as starting from $\u^0 =0$ and iteratively updating  $\u$ by 
\begin{align}\label{equ:ukupdates}
\u^k = (1-\xi_k)\u^{k-1} + \xi_k  \boldsymbol{q}^k, ~~~~~~~~
\boldsymbol{q}^k = \underset{\boldsymbol{q}\in \mathrm{Vert}(\M_N) }{\arg\min}\ell\left((1-\xi_k)\u^{k-1}+\xi_k\boldsymbol{q}\right),  
\end{align}
where $\mathrm{Vert}(\M_N):=\left\{ \bp(\th_{1}),...,\bp(\th_{N})\right\}$ denotes the vertices of $\M_N$, and we shall take  $\xi_k = 1/k$. 
We aim to prove that $\ell(\u^k) = O(1/k) +\DD_N^*$. 
Our proof can be easily extended to general convex functions $\ell(\cdot)$ and different $\xi_k$ schemes.  

By the convexity and the quadratic form of $\ell(\cdot)$,
for any $\boldsymbol{s}$, we have 
\begin{align}
& \ell(\boldsymbol{s})\ge\ell(\u^{k-1})+\nabla\ell(\u^{k-1})^{\top}(\boldsymbol{s}-\u^{k-1})\label{equ:convex1} \\    
& \ell(\boldsymbol{s})\le\ell(\u^{k-1})+\nabla\ell(\u^{k-1})^{\top}(\boldsymbol{s}-\u^{k-1}) +  \norm{\boldsymbol{s}-\u^{k-1}}^2.  \label{equ:convex2}
\end{align}

Minimizing $\boldsymbol{s}$ in $\M_{N}$ on both sides of \eqref{equ:convex1}, we have 
\begin{align} \label{equ:LNineq}
\begin{split}
\mathcal L_N^* = \min_{\boldsymbol{s} \in \M_N} \ell(\boldsymbol{s}) 
& \geq 
\min_{\boldsymbol{s} \in \mathcal \M_N}\left \{ \ell(\u^{k-1})+\nabla\ell(\u^{k-1})^{\top}(\boldsymbol{s} -\u^{k-1}) \right \} \\
& = 
 \ell(\u^{k-1})+
 \nabla\ell(\u^{k-1})^{\top}(\boldsymbol{s}^{k}-\u^{k-1}).
 \end{split}
\end{align}
Here we define 
\begin{align}\label{equ:fwupdate} 
\begin{split}
\boldsymbol{s}^k & = 
\underset{\boldsymbol{s}\in \M_N }{\arg\min}\nabla\ell(\u^{k-1})^{\top}(\boldsymbol{s} -\u^{k-1}) \\
& = \underset{\boldsymbol{s}\in\mathrm{Vert}(\M_N) 
}{\arg\min}\nabla\ell(\u^{k-1})^{\top}(\boldsymbol{s} -\u^{k-1}), \end{split}
\end{align}
where the second equation holds because we optimize a linear objective on a convex polytope $\M_N$ and hence the solution must be achieved on the vertices $\mathrm{Vert}(\M_N)$.   
Note that if we update $\u^k$ by $\u^k =(1-\xi_k)\u^{k-1} + \xi_k \boldsymbol{s}^k$, we would get the standard Frank-Wolfe (or conditional gradient) algorithm. The difference between our method and Frank-Wolfe is that we greedily minimize the loss $\ell(\u^k),$ while the Frank-Wolfe minimizes the linear approximation in \eqref{equ:fwupdate}.

Define $D_{\M_N}: = \max_{\u,\boldsymbol{v}}\{\norm{\u-\boldsymbol{v}} \colon ~ \u, \boldsymbol{v} \in \M_N\}$ to be the diameter of $\M_N$. 
Following \eqref{equ:ukupdates}, we have 
\begin{align}
\ell(\u^{k}) & =\underset{\boldsymbol{q}\in\mathrm{Vert}(\M_N) 
}{\min}\ell\left((1-\xi_k)\u^{k-1}+\xi_k\boldsymbol{q}\right)\notag \\
 & \le\ell\left((1-\xi_k)\u^{k-1}+\xi_k\boldsymbol{s}^{k}\right)\notag \\
 & \le\ell\left(\u^{k-1}\right)+\xi_k \nabla\ell(\u^{k-1})^\top\left(\boldsymbol{s}^{k}-\u^{k-1}\right) + C\xi_k^{2}  \label{equ:ww}\\
 & \le(1-\xi_k)\ell\left(\u^{k-1}\right)+\xi_k \DD_N^* +C\xi_k^{2},\label{equ:vffdfd}  
\end{align}
where we define $C :={D_{\M_N}^2}$, \eqref{equ:ww} follows \eqref{equ:convex2}, and \eqref{equ:vffdfd} follows \eqref{equ:LNineq}. 
Rearranging this, we get 
\[
\ell(\u^{k})-\DD_N^*-C\xi_k\le(1-\xi_k)\left(\ell(\u^{k-1})-\DD_N^*-C\xi_k\right)
\]
By iteratively applying the above inequality, we have
\begin{align*}
\ell(\u^{k})-\DD_N^*-C\xi_k & \le
\left( \prod_{i=1}^{k}(1-\xi_{i})\right ) \left(\ell(\u^{0})-\DD_N^*-C\xi_{1}\right).\\
\end{align*}
Taking $\xi_k = 1/k$. We get 
\begin{align*}
\ell(\u^{k})-\DD_N^*-\frac{C}{k} & \le
\frac{1}{k}\left(\ell(\u^{0})-\DD_N^*-C \right).\\
\end{align*}
And thus 
$$
\ell(\u^{k})\le   \frac{1}{k}\left(\ell(\u^{0})-\DD_N^* \right) + \DD_N^* = \mathcal O\left (\frac{1}{k} \right ) + \mathcal L_N^*.  
$$
This completes the proof.

\subsection{Proof of Theorem \ref{lem:Nnet}}

The proof leverages the idea from  the proof of Proposition~\ref{thm:convex} of   \citet{chen2012super} for analyzing their \emph{Herding} algorithm, but contains some extra nontrivial argument.

Following the proof of Proposition~\ref{thm:convex}, our problem can be viewed as  
$$
\min_{\u \in \M_N} \left \{\ell(\u) := \left\Vert \u-\y\right\Vert ^{2} \right\}, 
$$
with $\mathcal L_N^* = \min_{\u \in \M_N} \ell(\u)$,  our greedy algorithm can be viewed as starting from $\u^0 =0$ and iteratively updating  $\u$ by 
\begin{align}\label{equ:ukupdates}
\u^k = \frac{k-1}{k}\u^{k-1} + \frac{1}{k}  \boldsymbol{q}^k, ~~~~~~~~
\boldsymbol{q}^k = \underset{\boldsymbol{q}\in \mathrm{Vert}(\M_N) }{\arg\min}
\norm{\frac{k-1}{k}\u^{k-1}+\frac{1}{k}\boldsymbol{q} - \y }^2
\end{align}
where $\mathrm{Vert}(\M_N):=\left\{ \bp(\th_{1}),...,\bp(\th_{N})\right\}$ denotes the vertices of $\M_N$. 
We aim to prove that $\ell(\u^k) = \mathcal \red{O(1/(k\max(1,\gamma))^2)}$,  under Assumption~\ref{asm:inter_N}.  

Define $\w^k = k (\y - \u^{k})$, then $\ell(\u^k) = \norm{\w^k}^2/k^2$. Therefore, it is sufficient to prove that $\norm{\w^k} = \mathcal O(1/(\max(1,\gamma)))$. 

Similar to the proof of Proposition~\ref{thm:convex}, we define
\begin{align*} 
 \boldsymbol{s}^{k+1} & 
  = 
\underset{\boldsymbol{s}\in \M_N }{\arg\min}\nabla\ell(\u^{k})^{\top}(\boldsymbol{s} -\u^{k}) \\ 
& = \underset{\boldsymbol{s}\in \M_N }{\arg\min}\nabla\ell(\u^{k})^{\top}\boldsymbol{s} \\ 
& = \underset{\boldsymbol{s}\in \M_N }{\arg\min} \langle \w^k, ~~ \boldsymbol{s} \rangle. \\
& = \underset{\boldsymbol{s}\in \M_N }{\arg\min} \langle \w^k, ~~ (\boldsymbol{s} - \y)\rangle. 
\end{align*}
Because $\mathcal B(\y, \gamma)$ is included in $\M_N$ by Assumption~\ref{asm:inter_N}, we have $\boldsymbol{s}' := \y-\gamma\w^k/\norm{\w^k} \in \M_{N}$. Therefore 
$$
\langle \w^k, ~~ (\boldsymbol{s}^{k+1} - \y)\rangle  
= \min_{\boldsymbol{s}\in \M_N}\langle \w^k, ~~ (\boldsymbol{s} - \y)\rangle  
\leq \langle \w^k, ~~ (\boldsymbol{s}' - \y)\rangle   
= - \gamma \norm{\w^k}. 
$$

Note that
\begin{align*}
    \norm{\w^{k+1}}^2  & = \min_{\boldsymbol{q} \in \mathrm{Vert}(\M_N)}\norm{k \u^{k}  + \boldsymbol{q} - (k+1)\y}^2 \\
 & = \min_{\boldsymbol{q} \in \mathrm{Vert}(\M_N)}\norm{\w^{k}  + \boldsymbol{q} - \y}^2 \\    
    & \leq  \norm{\w^k +  \boldsymbol{s}^{k+1} - \y}^2 \\ 
    & = \norm{\w^k}^2 + 2 \langle \w^k, ~ (\boldsymbol{s}^{k+1} - \y ) \rangle + \norm{\boldsymbol{s}^{k+1} - \y}^2 \\
    & \leq \norm{\w^k}^2  - 2\gamma \norm{\w^k} + D_{\M_N}^2,
\end{align*}
where $D_{\M_N}$ is the diameter of $\M_N$.  
Because $\w^0 = 0$, 
using Lemma~\ref{lem:zzz}, 
we have
$$
\norm{\w^k}\leq \max({D_{\M_N}}, ~ D_{\M_N}^2/2,  D_{\M_N}^2/(2\gamma)) = \mathcal O\left ({\frac{1}{\min(1, \gamma)}}\right ),~~~~\forall k=1,2,\ldots, 
$$
This proves that $\ell(\u^k) = \frac{\norm{\w^k}^2}{k^2} = \mathcal O\left (\frac{1}{k^2\min(1,\gamma)^2} \right)$. 

\begin{lemma}\label{lem:zzz} 
Assume $\{z_k\}_{k\geq 0}$ is a sequence of numbers 
satisfying $z_0=0$ and 
$$
|z_{k+1}|^2 \leq |z_k|^2 - 2\gamma |z_k| + C,~~~~ \forall k = 0, 1,2, \ldots
$$
where $C$ and $\gamma$ are two positive numbers. Then we have $|z_k|\leq \max(\sqrt{C}, ~ C/2,  C/(2\gamma))$ for all $ k = 0, 1,2, \ldots$.  
\end{lemma}
\begin{proof}
We prove $|z_k|\leq \max(\sqrt{C}, ~ C/2,  C/(2\gamma)):= u_*$ by induction on $k$. 
Because $z_0=0$, the result holds for $k=0$.  
Assume   $|z_k|\leq u_*$, we want to prove that  $|z_{k+1}|\leq u_* 
$ also holds. 

Define $f(z) = z^2 - 2\gamma z + C$. 
Note that the maximum of $f(z)$ on an interval is always achieved on the vertices, because $f(z)$ is convex. 

\paragraph{Case 1:} If $|z_k| \leq C/(2\gamma)$, then we have
\begin{align*}
|z_{k+1}|^2 
& \leq  f(|z_k|) 
 \leq \max_{z} 
\bigg \{ f(z) \colon ~~ 
z \in [0,~~ C/(2\gamma)] \bigg\} 
 = \max\bigg \{ f(0), ~~  f(C/(2\gamma)) \bigg\} 
 =  \max\bigg \{C, ~~  C^2/(4\gamma^2) \bigg\} \leq u_*^2.  
\end{align*}

\paragraph{Case 2:} If $|z_k| \geq C/(2\gamma)$, then we have 
$$
|z_{k+1}|^2\leq |z_k|^2 - 2\gamma |z_k| + C  \leq |z_k|^2 \leq u_*^2.   
$$
In both cases, we have $|z_{k+1}|\leq u_*$. This completes the proof. 
\end{proof}

\subsection{Proof of Theorem \ref{thm:main}}
We first introduce the following Lemmas.


\begin{lemma} \label{lem:inter_mf}
Under the Assumption \ref{asm:bound_smooth},~\ref{asm:init},~\ref{asm:smooth_2},~\ref{asm:inter} and \ref{asm:density}. For any $\delta>0$, when $N$ is sufficient large, with probability
at least $1-\delta$,
\[
\mathcal{B}\left(\y,\frac{1}{2}\gamma^{*}\right)\subseteq\text{conv}\left\{ \bp(\th)\mid\th\in\supp(\rho_{T}^{N})\right\}.
\]
Here $\rho_{T}^{N}$ is the distribution of the weight of
the large network with $N$ neurons trained by gradient descent.
\end{lemma}

\subsubsection{Proof of Theorem \ref{thm:main}}
The above lemmas directly imply Theorem \ref{thm:main}.

\subsubsection{Proof of Lemma \ref{lem:inter_mf}}
In this proof, we simplify the statement that `for any $\delta>0$, when $N$ is sufficiently large, event  $E$ holds with probability at least $1-\delta$' by `when $N$ is sufficiently large, with high probability, event $E$ holds'.

By the Assumption~\ref{asm:inter},  there exists $\gamma^{*}>0$ such that
\[
\mathcal{B}\left(\y,\gamma^{*}\right)\subseteq\text{conv}\left\{ \bp(\th)\mid\th\in\supp(\rho_{T}^{\infty})\right\} = \mathcal M.
\]


Given
any $\th\in\supp(\rho_{T}^{\infty})$, define 
\[
\bp^{N}\left(\th\right)=\underset{\th'\in\supp\left(\rho_{T}^{N}\right)}{\arg\min}\left\Vert \bp(\th')-\bp(\th)\right\Vert 
\]
where $\bp^{N}(\th)$ is the best approximation of $\bp(\th)$ using the points $\bp(\th_{i}),\th_{i}\in\supp(\rho_{T}^{N})$.

Using Lemma \ref{techlem:neighbor3}, by choosing $\epsilon=\gamma^*/6$,
when $N$ is sufficiently
large, we have 
\begin{equation}\label{eq:appendix_lemma5_eq1}
\underset{\th\in\supp(\rho_{T}^{\infty})}{\sup}\left\Vert \bp(\th)-\bp^{N}(\th)\right\Vert \le\gamma^*/6,
\end{equation}
with high probability. \eqref{eq:appendix_lemma5_eq1} implies that $\M_N$ can approximate $\M$ for large $N$. Since $\M$ is assumed to contain the ball centered at $\y$ with radius $\gamma^*$, as $\M_N$ approximates $\M$, intuitively $\M_N$ would also contain the ball centered at $\y$ with a smaller radius.
And below we give a rigorous proof for this intuition.

\paragraph{Step 1: $\left\Vert \hat{\y}-\y\right\Vert \le\gamma^{*}/6$.}
When $N$ is sufficiently large, with high probability, we have
\begin{align*}
 \left\Vert \hat{\y}-\y\right\Vert
\le \sum_{i=1}^{M}q_{i}\left\Vert \bp^{N}(\th_{i}^{*})-\bp(\th_{i}^{*})\right\Vert \le\gamma^{*}/6.
\end{align*}

\paragraph{Step 2 $\mathcal{B}\left(\hat{\protect\y},\frac{5}{6}\gamma^{*}\right)\subseteq\protect\M$}

By step one, with high probability, $\left\Vert \hat{\y}-\y\right\Vert \le\gamma^{*}/4$,
which implies that $\hat{\y}\in\mathcal{B}\left(\y,\gamma^{*}/4\right)\subseteq\mathcal{B}\left(\y,\gamma^{*}\right)\subseteq\M$.
Also, for any $A\in\partial\M$ (here $\partial\M$ denotes the boundary
of $\M$), we have 
\[
\left\Vert \hat{\y}-A\right\Vert \ge\left\Vert \y-A\right\Vert -\left\Vert \y-\hat{\y}\right\Vert \ge\gamma^{*}-\gamma^{*}/4.
\]
This gives that $\mathcal{B}\left(\hat{\y},\frac{5}{6}\gamma^{*}\right)\subseteq\M$.

\paragraph{Step 3 $\mathcal{B}\left(\hat{\protect\y},\frac{2}{3}\gamma^{*}\right)\subseteq\protect\M_{N}$}

Notice that $\hat{\y}$ is a point in $\R^{m}$ and suppose that $A$
belongs to the boundary of $\M_{N}$ (denoted by $\partial\M_{N}$)
such that 
\[
\left\Vert \hat{\y}-A\right\Vert =\min_{\tilde{A}\in\partial\M_{N}}\left\Vert \hat{\y}-\tilde{A}\right\Vert .
\]
We prove by contradiction. Suppose that we have $\left\Vert \hat{\y}-A\right\Vert <\frac{2}{3}\gamma^{*}$. 

Using support hyperplane theorem, there exists a hyperplane $P=\{\boldsymbol{u}:\left\langle \boldsymbol{u}-A,\boldsymbol{v}\right\rangle =0\}$
for some nonempty vector $\boldsymbol{v}$, such that $A\in P$ and
\[
\sup_{\boldsymbol{q}\in\M_{N}}\left\langle \boldsymbol{q},\boldsymbol{v}\right\rangle \le\left\langle A,\boldsymbol{v}\right\rangle .
\]
We choose $A'\in P$ such that $A'-\hat{\y}\perp P$ ($A$ and $A'$
can be the same point). Notice that 
\[
\left\Vert \hat{\y}-A'\right\Vert ^{2}=\left\Vert \hat{\y}-A+A-A'\right\Vert ^{2}=\left\Vert \hat{\y}-A\right\Vert ^{2}+\left\Vert A-A'\right\Vert ^{2}+2\left\langle \hat{\y}-A,A-A'\right\rangle .
\]
Since $A'-\hat{\y}\perp P$ and $A,A'\in P$, we have $\left\langle \hat{\y}-A,A-A'\right\rangle =0$
and thus $\left\Vert \hat{\y}-A'\right\Vert \le\left\Vert \hat{\y}-A\right\Vert <\frac{2}{3}\gamma^{*}$.
We have 
\[
A'\in\mathcal{B}\left(\hat{\y},\left\Vert \hat{\y}-A\right\Vert \right)\subseteq\mathcal{B}\left(\hat{\y},\frac{2}{3}\gamma^{*}\right)\subseteq\mathcal{B}\left(\hat{\y},\frac{5}{6}\gamma^{*}\right)\subseteq\M.
\]
Notice that as both $\hat{\y},A'\in\M$ we choose $\lambda\ge1$ such
that $\hat{\y}+\lambda\left(A'-\hat{\y}\right)\in\partial\M$, where
$\partial\M$ denotes the boundary of $\M$. Define $B=\hat{\y}+\lambda\left(A'-\hat{\y}\right)$.
As we have shown that $\mathcal{B}\left(\hat{\y},\frac{5}{6}\gamma^{*}\right)\subseteq\M$,
we have $\left\Vert \hat{\y}-B\right\Vert \ge\frac{5}{6}\gamma^{*}$.
And thus
\begin{align*}
\left\Vert B-A'\right\Vert  & =\left\Vert B-\hat{\y}\right\Vert -\left\Vert \hat{\y}-A'\right\Vert \\
 & >\frac{5}{6}\gamma^{*}-\frac{2}{3}\gamma^{*}\\
 & >\frac{1}{6}\gamma^{*}.
\end{align*}
 Also notice that 
\begin{align*}
\left\langle B-A,\boldsymbol{v}\right\rangle  & =\left\langle \hat{\y}+\lambda\left(A'-\hat{\y}\right)-A,\boldsymbol{v}\right\rangle \\
 & =(1-\lambda)\left\langle \hat{\y}-A,\boldsymbol{v}\right\rangle +\lambda\left\langle A'-A,\boldsymbol{v}\right\rangle \\
 & =(1-\lambda)\left\langle \hat{\y}-A,\boldsymbol{v}\right\rangle \\
 & \ge0.
\end{align*}
This implies that $B$ and $\M$ are on different side of $P$.

With high probability, we are able to find $D\in\{\bp(\th);\th\in\supp(\rho_{T}^{N})\}$
such that 
\[
\left\Vert D-B\right\Vert \le\frac{\gamma^{*}}{6}.
\]
By the definition, $D\in\M_{N}$ and thus $\left\langle D-A,\boldsymbol{v}\right\rangle \le0$
as shown by the supporting hyperplane theorem. Also remind that $\left\langle B-A,\boldsymbol{v}\right\rangle \ge0$.
These allow us to choose $\lambda'\in[0,1]$ such that 
\[
\left\langle \lambda'D+(1-\lambda')B-A,\boldsymbol{v}\right\rangle =0.
\]
We define $E=\lambda'D+(1-\lambda')B$ and thus $E\in P$. Notice
that 
\[
\left\Vert B-E\right\Vert =\left\Vert B-\lambda'D-(1-\lambda')B\right\Vert =\lambda'\left\Vert B-D\right\Vert \le\left\Vert B-D\right\Vert \le\frac{\gamma^{*}}{6}.
\]
Also, 
\[
\left\Vert B-E\right\Vert ^{2}=\left\Vert B-A'+A'-E\right\Vert ^{2}=\left\Vert B-A'\right\Vert ^{2}+\left\Vert A'-E\right\Vert ^{2}+2\left\langle B-A',A'-E\right\rangle .
\]
As $B-A'\perp P$ and $A',E\in P$, we have $\left\langle B-A',A'-E\right\rangle =0$,
which implies that $\left\Vert B-E\right\Vert \ge\left\Vert B-A'\right\Vert >\frac{1}{6}\gamma^{*}$,
which makes contradiction.

\paragraph{Step 4 $\mathcal{B}\left(\protect\y,\frac{1}{2}\gamma^{*}\right)\subseteq\protect\M_{N}$}

As for sufficiently large $N$, we have $\left\Vert \hat{\y}-\y\right\Vert \le\frac{1}{6}\gamma^{*}$
and thus 
\[
\mathcal{B}\left(\y,\frac{1}{2}\gamma^{*}\right)\subseteq\mathcal{B}\left(\hat{\y},\frac{2}{3}\gamma^{*}\right)\subseteq\M_{N}.
\]

\section{Technical Lemmas}

\begin{lemma} \label{techlem:bound}
Under assumption \ref{asm:bound_smooth} and \ref{asm:init}, for any $N$, at training time $T<\infty$, for any $\th\in\supp(\rho_{T}^{N})$ or $\th\in\supp(\rho_{T}^{\infty})$, we have $\left\Vert \th\right\Vert \le C$, $\left\Vert \bp(\th)\right\Vert \le C$ and $\left\Vert \bp(\th)\right\Vert _{\text{Lip}}\le C$ for some constant $C<\infty$. 
\end{lemma}

\begin{lemma} \label{techlem:neighbor}
Suppose $\th_{i}\in\R^{d}$, $i=1,...,N$ are i.i.d. samples from
some distribution $\rho$ and $\Omega\subseteq\R^{d}$ is bounded.
For any radius $r_{B}>0$ and $\delta>0$, define the following two
sets
\begin{align*}
A & =\left\{ \th_{B}\in\Omega\, \middle|\, \mathbb{P}_{\th\sim\rho}\left(\th\in \mathcal{B}\left(\th_B,r_B\right)\right)>\frac{4}{N}\left(\left(d+1\right)\log\left(2N\right)+\log\left(8/\delta\right)\right)\right\} \\
B & =\left\{ \th_{B}\in\Omega\,\middle|\,\ \left\Vert \th_{B}-\th_{B}^{N}\right\Vert \le r_{B}\right\} ,
\end{align*}
where $\th_{B}^{N}=\underset{\th'\in\{\th_{i}\}_{i=1}^{N}}{\arg\min}\left\Vert \th_{B}-\th'\right\Vert .$
With probability at least $1-\delta$, $A\subseteq B$.
\end{lemma}

\begin{lemma} \label{techlem:neighbor2}
For any $\delta>0$ and $\epsilon>0$, when $N$ is sufficiently large ($N$ depends on $\delta$), with probability
at least $1-\delta$, we have 
\[
\underset{\th\in\supp(\rho_{T}^{\infty})}{\sup}\left\Vert \bp(\th)-\bar{\bp}^{N}(\th)\right\Vert \le\epsilon,
\]
where $\bar{\bp}^{N}(\th)=\underset{\bp(\bar{\th}')\in\{\bp(\bar{\th}_{i})\}_{i=1}^{N}}{\arg\min}\left\Vert \bp(\bar{\th}')-\bp(\th)\right\Vert $ and $\bar{\th}_i$ are i.i.d. samples from $\rho_T^\infty$.
\end{lemma}

\begin{lemma} \label{techlem:neighbor3}
For any $\delta>0$ and $\epsilon>0$, when $N$ is sufficiently large ($N$ depends on $\delta$), with probability
at least $1-\delta$, we have 
\[
\underset{\th\in\supp(\rho_{T}^{\infty})}{\sup}\left\Vert \bp(\th)-\bp^{N}(\th)\right\Vert \le\epsilon,
\]
where $\bp^{N}\left(\th\right)=\underset{\th'\in\supp\left(\rho_{T}^{N}\right)}{\arg\min}\left\Vert \bp(\th')-\bp(\th)\right\Vert$.
\end{lemma}

\subsection{Proof of Lemma \ref{techlem:bound}}

We prove the case of training network with $N$ neurons. Notice that
\begin{align*}
\left\Vert \frac{\partial}{\partial t}\th(t)\right\Vert  & =\left\Vert \grad[\th(t),\rho_{t}^{N}]\right\Vert \\
 & =\left\Vert \E_{\x,y\sim\D}\left(y-f_{\rho_{t}^{N}}(\x)\right)\nabla_{\th}\sigma(\th(t),\x)\right\Vert \\
 & \le\sqrt{\E_{\x,y\sim\D}\left(y-f_{\rho_{t}^{N}}(\x)\right)^{2}}\sqrt{\E_{\x,y\sim\D}\left\Vert \nabla_{\th}\sigma(\th(t),\x)\right\Vert ^{2}}\\
 & \le\sqrt{\E_{\x,y\sim\D}\left(y-f_{\rho_{0}^{N}}(\x)\right)^{2}}\sqrt{\E_{\x,y\sim\D}\left\Vert \nabla_{\th}\sigma(\th(t),\x)\right\Vert ^{2}}
\end{align*}
Notice that by the assumption \ref{asm:bound_smooth}, we have $\sqrt{\E_{\x,y\sim\D}\left(y-f_{\rho_{0}^{N}}(\x)\right)^{2}}\le C$.
Remind that $\th(t)=[\boldsymbol{a}(t),b(t)]$, $\sigma(\th(t),\x)=b(t)\sigma_{+}(\boldsymbol{a}^\top(t)\x)$.
Thus we have
\[
\left|\frac{\partial}{\partial t}b(t)\right|\le C\left\Vert \sigma_{+}\right\Vert _{\infty}.
\]
And thus for any $i\in\{1,...,N\}$, $\underset{t\in[0,T]}{\sup}\left\Vert b_{i}(t)\right\Vert \le\int_{0}^{T}\left\Vert \frac{\partial}{\partial t}b_{i}(s)\right\Vert ds\le TC.$
Also
\begin{align*}
\left\Vert \frac{\partial}{\partial t}\boldsymbol{a}(t)\right\Vert  & \le C|b(t)|\left\Vert \sigma_{+}'\right\Vert _{\infty}\sqrt{\E_{\x\sim\mathcal{D}}\left\Vert \x\right\Vert ^{2}}\\
 & \le TC.
\end{align*}
By assumption \ref{asm:init}, that $\left\Vert \th_{0}(t)\right\Vert \le C$,
we have 
\[
\underset{t\in[0,T]}{\sup}\left\Vert \th_{i}(t)\right\Vert \le\int_{0}^{T}\left\Vert \frac{\partial}{\partial t}\th_{i}(s)\right\Vert ds\le T^{2}C.
\]
 Notice that this also holds to training the network with infinite number of neurons. Notice that
$\left\Vert \bp(\th)\right\Vert =\sqrt{\frac{1}{m}\sum_{j=1}^{m}\sigma^{2}(\th,\x^{(j)})}\le CT$. And
\begin{align*}
\left\Vert \bp(\th)\right\Vert _{\text{Lip}} & =\sup_{\th_1,\th_2}\frac{\left\Vert \bp(\th_{1})-\bp(\th_{2})\right\Vert }{\left\Vert \th_{1}-\th_{2}\right\Vert }\\
 & =\sup_{\th_1,\th_2}\frac{\sqrt{\frac{1}{m}\sum_{j=1}^{m}\left(\sigma(\th_{1},\x^{(j)})-\sigma(\th_{2},\x^{(j)})\right)^{2}}}{\left\Vert \th_{1}-\th_{2}\right\Vert }\\
 & \le TC\left\Vert \sigma_{+}\right\Vert _{\lip}+\left\Vert \sigma_{+}\right\Vert _{\infty}.
\end{align*}
Thus given any $T<\infty$, all those three quantities can be bounded
by some constant.

\subsection{Proof of Lemma \ref{techlem:neighbor}}

The following proof follows line 1 and and line 2 of the proof of Lemma 16 of \cite{chaudhuri2010rates}.

Define $g_{\th_{B}}(\th)=\mathbb{I}\left\{ \th\in\mathcal{B}\left(\th_{B},r_{B}\right)\right\} $
and $\beta_{N}=\sqrt{(4/N)(d_{\text{VC}}\log2N+\log(8/\delta))}$,
where $d_{\text{VC}}$ is the VC dimension of the function class $\mathcal{G}=\{g_{\th_{B}},\th_{B}\in\Omega\}$
and thus $d_{\text{VC}}\leq d+1$ \cite{dudley1979balls}. 
Let $\mathbb{E}g_{\th_{B}}=\mathbb{P}_{\th\sim\rho}\left(\th\in\mathcal{B}\left(\th_{B},r_{B}\right)\right)$ and $\mathbb{E}_N g_{\th_{B}}=\sum_{i=1}^{N}g_{\th_{B}}(\th_i)/N$.
So
$$
A = \{ \th_{B}\, |\, \mathbb{E}g_{\th_{B}}>\beta_N^2 \}
$$
and we further define
\begin{equation*}
A_2 =\left\{ \th_{B}\, |\, \mathbb{E}_N g_{\th_{B}} > 0 \right\} .
\end{equation*}

From theorem 15 of \cite{chaudhuri2010rates} (which is a rephrase of the generalization bound), we know that: for any $\delta>0$, with probability at least $1-\delta$, the following holds for all $g_{\th_{B}}\in\mathcal{G}$, 
\begin{equation}\label{eq: lemma8}
    \mathbb{E}g_{\th_{B}}-\mathbb{E}_N g_{\th_{B}}\leq \beta_N \sqrt{\mathbb{E}g_{\th_{B}}}
\end{equation}
Notice that for any $g_{\th_{B}}$ which satisfies \eqref{eq: lemma8}, 
$$
\mathbb{E}g_{\th_{B}}>\beta_N^2\Rightarrow \mathbb{E}_N g_{\th_{B}} > 0
$$
So this means: for any $\delta>0$, with probability at least $1-\delta$,
\begin{equation*}
    A\subseteq A_2=B
\end{equation*}
where the last equality follows from the following:
\begin{equation*}
\begin{split}
    A_2 &=\left\{ \th_{B}\, \middle|\, \mathbb{E}_N g_{\th_{B}} > 0 \right\}
    =\left\{\text{there exists some}\,\, \th_i\,\, \text{such that}\,\, \th_i\in\mathcal{B}(\th_B,r_B)\right\}
    =B
\end{split}   
\end{equation*}

\subsection{Proof of Lemma \ref{techlem:neighbor2}}
Given $\epsilon>0$, we choose $r_{0}$ sufficiently small such that
$Cr_{0}\le\epsilon$ (here $C$ is some constant defined in Lemma \ref{techlem:bound}). For this choice of $r_0$, given the corresponding $p_{0}$ (defined in assumption \ref{asm:density}), for any $\delta>0$,
there exists $N(\delta)$ such that $\forall N\ge N(\delta)$, we
have 
\[
p_{0}>\frac{4}{N}\left((d+1)\log(2N)+\log(8/\delta)\right):= \beta_N^2.
\]
And thus from assumption \ref{asm:density}, we have 
\[
\forall\th\in\supp(\rho_{T}^{\infty}),\ \mathbb{P}_{\th'\sim\rho_{T}^{\infty}}\left(\th'\in\mathcal{B}(\th,r_{0})\right)\ge p_{0}>\beta_{N}^{2}.
\]
This implies
\begin{equation*}
\text{supp}(\rho_T^\infty)    \subseteq A =\left\{ \th_{B}\, |\, \mathbb{P}_{\th\sim\rho}\left(\th\in \mathcal{B}\left(\th_B,r_0\right)\right)>\beta_N^2\right\} 
\end{equation*}

From Lemma \ref{techlem:neighbor} (set $r_B=r_0$), we know:
with probability at least $1-\delta$,
\begin{equation*}
A 
\subseteq
B =\left\{ \th_{B}\in\Omega\,\middle|\,\ \left\Vert \th_{B}-\th_{B}^{N}\right\Vert \le r_{0}\right\} ,
\end{equation*}
Thus, with probability at least $1-\delta$, 
$$\supp(\rho_{T}^{\infty})\subseteq B$$
and this means: with probability at least $1-\delta$,
we have 
\[
\forall\th\in\supp(\rho_{T}^{\infty}),\ \left\Vert \th-\th^{N}\right\Vert \le r_{0}.
\]
The result concludes from
\begin{align*}
 & \underset{\th\in\supp(\rho_{T}^{\infty})}{\sup}\left\Vert \bp(\th)-\bp^{N}(\th)\right\Vert \\
\le & \underset{\th\in\supp(\rho_{T}^{\infty})}{\sup}\left\Vert \bp(\th)-\bp(\th^{N})\right\Vert \\
\le & \underset{\th\in\supp(\rho_{T}^{\infty})}{\sup}C\left\Vert \th-\th^{N}\right\Vert \\
\le & Cr_{0} \le \epsilon.
\end{align*}
Here the last inequality uses Lemma \ref{techlem:bound}.

\subsection{Proof of Lemma \ref{techlem:neighbor3}}
In this proof, we simplify the statement that `for any $\delta>0$, when $N$ is sufficiently large, event  $E$ holds with probability at least $1-\delta$' by `when $N$ is sufficiently large, with high probability, event $E$ holds'.

Suppose that $\th_{i}$, $i\in[N]$ is the weight of neurons of network
$f_{\rho_{T}^{N}}$. Given any $\th\in\supp(\rho_{T}^{\infty})$,
define 
\[
\bp^{N}(\th)=\underset{\bp(\th')\in\mathrm{Vert}(\M_{N})}{\arg\min}\left\Vert \bp(\th')-\bp(\th)\right\Vert .
\]
Notice that the training dynamics of the network with $N$ neurons
can be characterized by 
\begin{align*}
\frac{\partial}{\partial t}\th_{i}(t) & =\grad[\th_{i}(t),\rho_{t}^{N}],\\
\th_{i}(0) & \overset{\text{i.i.d.}}{\sim}\rho_{0}.
\end{align*}
Here $\grad[\th,\rho]=\E_{\x,y\sim\D}\left(y-f_{\rho}(\x)\right)\nabla_{\th}\sigma(\th,\x).$ We define the following coupling dynamics:
\begin{align*}
\frac{\partial}{\partial t}\bar{\th}_{i}(t) & =\grad[\bar{\th}_{i}(t),\rho_{t}^{\infty}],\\
\bar{\th}_{i}(0) & =\th_{i}(0).
\end{align*}
Notice that at any time $t$, $\bar{\th}_{i}(t)$ can be viewed as
i.i.d. sample from $\rho_{t}^{\infty}$. We define $\hat{\rho}_{t}^{N}(\th)=\frac{1}{N}\sum_{i=1}^{N}\delta_{\bar{\th}_{i}(t)}(\th)$.
Notice that by our definition $\th_{i}=\th_{i}(T)$ and we also define
$\bar{\th}_{i}=\bar{\th}_{i}(T)$. Using the propagation of chaos
argument as \citet{mei2019meandimfree} (Proposition 2 of Appendix B.2), for any $T<\infty$,
for any $\delta>0$, we have 
\[
\underset{t\in[0,T]}{\sup}\underset{i\in\{1,..,N\}}{\max}\left\Vert \bar{\th}_{i}(t)-\th_{i}(t)\right\Vert \le\frac{C}{\sqrt{N}}\left(\sqrt{\log N}+\sqrt{\log1/\delta}\right).
\]

By Lemma \ref{techlem:neighbor2} and the bound above, when $N$ is sufficiently large, with
high probability, we have 
\begin{align*}
\underset{\th\in\supp(\rho_{T}^{\infty})}{\sup}\left\Vert \bp(\th)-\bar{\bp}^{N}(\th)\right\Vert  & \le\epsilon/2\\
\max_{i\in[N]}\left\Vert \bar{\th}_{i}(T)-\th_{i}(T)\right\Vert  & \le\frac{\epsilon}{2C},
\end{align*}
where $C=\left\Vert \bp\right\Vert _{\lip}$ and
\[
\bar{\bp}^{N}(\th)=\underset{\th'\in\supp(\hat{\rho}_{T}^{N})}{\arg\min}\left\Vert \bp(\th)-\bar{\bp}^{N}(\th)\right\Vert .
\]
We denote $\bar{\th}_{i_{\th}}\in\supp(\hat{\rho}_{T}^{N})$ such
that $\bar{\bp}^{N}(\th)=\bp(\bar{\th}_{i_{\th}})$. It implies that
\begin{align*}
\underset{\th\in\supp(\rho_{T}^{\infty})}{\sup}\left\Vert \bp(\th)-\bp^{N}(\th)\right\Vert  & \le\underset{\th\in\supp(\rho_{T}^{\infty})}{\sup}\left\Vert \bp(\th)-\bp\left(\th_{i_{\th}}\right)\right\Vert \\
 & =\underset{\th\in\supp(\rho_{T}^{\infty})}{\sup}\left\Vert \bp(\th)-\bar{\bp}^{N}(\th)+\bar{\bp}^{N}(\th)-\bp\left(\th_{i_{\th}}\right)\right\Vert \\
 & =\underset{\th\in\supp(\rho_{T}^{\infty})}{\sup}\left\Vert \bp(\th)-\bp\left(\bar{\th}_{i_{\th}}\right)+\bp\left(\bar{\th}_{i_{\th}}\right)-\bp\left(\th_{i_{\th}}\right)\right\Vert \\
 & \le\underset{\th\in\supp(\rho_{T}^{\infty})}{\sup}\left\Vert \bp(\th)-\bp\left(\bar{\th}_{i_{\th}}\right)\right\Vert +\underset{\th\in\supp(\rho_{T}^{\infty})}{\sup}\left\Vert \bp(\th)-\bp\left(\bar{\th}_{i_{\th}}\right)\right\Vert \\
 & \le\epsilon/2+\max_{i\in[N]}\left\Vert \bar{\th}_{i}(T)-\th_{i}(T)\right\Vert \left\Vert \bp\right\Vert _{\lip}\\
 & \le\epsilon.
\end{align*}

\end{document}